\documentclass{article}

\usepackage{microtype}
\usepackage{graphicx}
\usepackage{subfigure}
\usepackage{booktabs} %

\usepackage{algorithm}
\usepackage{algorithmic}

\usepackage{amsmath}
\usepackage{amssymb}
\usepackage{mathtools}
\usepackage{amsthm}

\usepackage[capitalize]{cleveref}

\usepackage{thmtools}
\declaretheorem[name=Theorem,refname={Theorem,Theorems},Refname={Theorem,Theorems}]{theorem}
\declaretheorem[name=Lemma,refname={Lemma,Lemmas},Refname={Lemma,Lemmas},sibling=theorem]{lemma}
\declaretheorem[name=Corollary,refname={Corollary,Corollaries},Refname={Corollary,Corollaries},sibling=theorem]{corollary}

\declaretheorem[name=Proposition,refname={Proposition,Propositions},Refname={Proposition,Propositions},sibling=theorem]{proposition}
\declaretheorem[name=Definition,refname={Definition,Definitions},Refname={Definition,Definitions},sibling=theorem]{definition}

\usepackage[disable]{todonotes}

\usepackage{amsbsy}

\PassOptionsToPackage{usenames,dvipsnames}{xcolor} %

\newcommand{\todob}[2][]{\todo[color=red!20!white,size=\scriptsize,inline,#1]{B: #2}} %
\newcommand{\jtodo}[2][]{\todo[color=blue!20!white,size=\scriptsize,#1]{J: #2}} %
\newcommand{\jnote}[2][]{\todo[color=blue!20!white,size=\small,#1]{J: #2}} %

\newcommand{\abs}[1]{\left|#1\right|}                   %
\newcommand{\vnormf}[1]{\|#1\|}                         %
\newcommand{\tr}[1]{\mbox{tr}#1}

\newcommand{\E}{\mathbb{E}}

\renewcommand{\P}{\mathbf{P}}
\newcommand{\R}{\mathbb{R}} %
\DeclareMathOperator*{\argmax}{\arg\max}

\newcommand{\A}{\mathcal{A}}

\newcommand{\ca}[1]{\mathcal{#1}}

\newcommand{\T}{^\top}

\newcommand{\diam}{\mathrm{diam}}

\newcommand{\hatTheta}{\hat{\theta}_* := \Vmi\sum_{s=1}^{m_0}\sum_{i=1}^da_iy_{s,i}}
\newcommand{\VmDef}{V_{m_0}:={m_0}\sum_{i=1}^d a_ia_i\T}

\usepackage{upgreek}

\usepackage{thmtools,thm-restate}
\usepackage{chngcntr}
\usepackage{apptools}

\usepackage{pgfplotstable}
\pgfplotsset{compat=1.17}

\usepackage{xspace}

\newcommand{\Alg}{{\tt alg}\xspace} %
\newcommand{\alg}{{\tt alg}\xspace} %
\newcommand{\expl}{{\tt explore}\xspace}
\newcommand{\tv}{\mathrm{TV}\xspace}

\newcommand{\SR}{\mathrm{SR}\xspace}
\newcommand{\BSR}{\mathrm{BSR}\xspace}
\newcommand{\parr}{~||~}
\newcommand{\Ahat}{\hat{A}}
\newcommand{\I}[1]{\mathrm{I}\left(#1\right)}
\newcommand{\Info}{\mathrm{I}}
\newcommand{\da}{\mathrm{d}}

\newcommand{\Vmi}{V_{m_0}^{-1}}
\newcommand{\Indic}{\textbf{1}}

\newcommand{\TS}{{\tt TS}\xspace}
\newcommand{\ucb}{{BayesUCB}\xspace}
\newcommand{\KL}[1]{\mbox{KL}#1}
\newcommand{\Ngaus}{\mathcal{N}}
\newcommand{\Spn}{\mathrm{Span}}

\usepackage{enumitem}

\usepackage{tikz}
\usetikzlibrary{positioning, backgrounds,
                fit}
                
\tikzstyle{block} = [draw=black, thick, text width=1.7cm, minimum height=1.1cm, align=center, fill=none]  

\tikzstyle{arrow} = [thick,->,>=stealth]

\newcommand{\bMetaSRM}{{\tt B-metaSRM}\xspace}
\newcommand{\fMetaSRM}{{\tt f-metaSRM}\xspace}
\newcommand{\MisbMetaSRM}{{\tt MisB-metaSRM}\xspace}
\newcommand{\OTS}{{\tt OracleTS}\xspace}
\newcommand{\SH}{{\tt SH}\xspace} %
\newcommand{\LinSH}{{\tt Lin-SH}\xspace} %
\newcommand{\GSE}{{\tt GSE}\xspace} %
\newcommand{\LinGapE}{{\tt LinGapE}\xspace} %

\usepackage{multicol}

\newcommand{\MLin}{{m_0}\geq \left(\frac{  d \log(2/\delta) \sum_{i=1}^d\sigma_i^2 }{2\sigma_0\lambda_d^{4}(\sum_{i=1}^d a_i a_i\T) \epsilon^4}\right)^{1/3}}

\newcommand{\MBern}{{m_0}\geq \frac{C|\A|^2\log(|\A|/\delta)}{\epsilon^2}}

\newcommand{\tnorm}{\lambda_d^{-1}\left(\sum_{i=1}^d a_i a_i\T\right) \left(\frac{d}{M^3}\log(2/\delta) \sum_{i=1}^d\sigma_i^2\right)^{1/4}}
\usepackage[round,sectionbib]{natbib}
\usepackage{amsbsy}
\usepackage{url}

\date{}
\author{MohammadJavad Azizi \\ 
{\small University of Southern California}\\
{\small\tt azizim@usc.edu} 
\and
Branislav Kveton \\ 
{\small Amazon}
\\{\small\tt bkveton@amazon.com}
\and
Mohammad Ghavamzadeh \\ 
{\small Google Research}
\\{\small\tt ghavamza@google.com} 
\and
Sumeet Katariya \\ 
{\small Amazon}
\\{\small\tt  katsumee@amazon.com} 
}
\date{}  
\begin{document}

\title{\bf Meta-Learning for Simple Regret Minimization}
\maketitle

\begin{abstract}
We develop a meta-learning framework for simple regret minimization in bandits. In this framework, a learning agent interacts with a sequence of bandit tasks, which are sampled i.i.d.\ from an unknown prior distribution, and learns its meta-parameters to perform better on future tasks. We propose the first Bayesian and frequentist meta-learning algorithms for this setting. The Bayesian algorithm has access to a prior distribution over the meta-parameters and its meta simple regret over $m$ bandit tasks with horizon $n$ is mere $\tilde{O}(m / \sqrt{n})$. On the other hand, the meta simple regret of the frequentist algorithm is $\tilde{O}(\sqrt{m} n + m/ \sqrt{n})$. While its regret is worse, the frequentist algorithm is more general because it does not need a prior distribution over the meta-parameters. It can also be analyzed in more settings. We instantiate our algorithms for several classes of bandit problems. Our algorithms are general and we complement our theory by evaluating them empirically in several environments.
\end{abstract}

\section{Introduction}
\label{sec:introduction}

We study the problem of \emph{simple regret minimization (SRM)} in a \emph{fixed-horizon (budget) setting} \citep{audibert-2010-BAI,kaufmann16complexity}. The learning agent interacts sequentially with $m$ such tasks, where each task has a horizon of $n$ rounds. The tasks are sampled i.i.d.\ from a prior distribution $P_*$, which makes them similar. We study a meta-learning \citep{thrun96explanationbased,thrun98lifelong,baxter98theoretical,baxter00model} variant of the problem, where the prior distribution $P_*$ is unknown, and the learning agent aims to learn it to reduce its regret on future tasks.

This problem is motivated by practical applications, such as online advertising, recommender systems, hyper-parameter tuning, and drug repurposing \citep{pmlr-v33-hoffman14,NEURIPS2020_edf0320a,reda2021top, pmlr-v139-alieva21a}, where bandit models are popular due to their simplicity and efficient algorithms. These applications include a test phase separated from the commercialization phase, and one aims at minimizing the regret of the commercialized product (simple regret) rather than the cumulative regret in the test phase \citep{audibert-2010-BAI}. In all of these, the exploration phase is limited by a fixed horizon: the budget for estimating click rates on ads is limited, or a hyper-parameter tuning task has only a limited amount of resources \citep{pmlr-v139-alieva21a}. Meta-learning can result in more efficient exploration when the learning agent solves similar tasks over time.

To understand the benefits of meta-learning, consider the following example. Repeated A/B tests are conducted on a website to improve customer engagement. Suppose that the designers always propose a variety of website designs to test. However, dark designs tend to perform better than light ones, and thus a lot of customer traffic is repeatedly wasted to discover the same pattern. One solution to reducing waste is that the designers to stop proposing light designs. However, these designs are sometimes better. A more principled solution is to automatically adapt the prior $P_*$ in A/B tests to promote dark designs unless proved otherwise by evidence. This is the key idea in the proposed solution in this work.

We make the following contributions. First, we propose a general meta-learning framework for fixed-horizon SRM in \cref{sec:Setup}. While several recent papers studied this problem in the cumulative regret setting \citep{bastani19meta,cella20metalearning,kveton2021metathompson,basu2021no,simchowitz2021bayesian}, this work is the first application of meta-learning to SRM. We develop general Bayesian and frequentist algorithms for this problem in \cref{sec:BMBAI,sec:fMBAI}. Second, we show that our Bayesian algorithm, which has access to a prior over the meta-parameters of $P_*$, has meta simple regret $\tilde{O}(m / \sqrt{n})$ over $m$ bandit tasks with horizon $n$. Our frequentist algorithm is more general because it does not need a prior distribution over the meta-parameters. However, we show that its meta simple regret is $\tilde{O}(\sqrt{m} n + m / \sqrt{n})$, and thus, worse than that of the Bayesian algorithm. In \cref{sec:fMBAI-LB}, we present a lower bound showing that this is unimprovable in general. Third, we instantiate both algorithms in multi-armed and linear bandits in \cref{sec:examples}. These instances highlight the trade-offs of the Bayesian and frequentist approaches, a provably lower regret versus more generality. Finally, we complement our theory with experiments (\cref{sec:experiment}), which show the benefits of meta-learning and confirm that the Bayesian approaches are superior whenever implementable.

Some of our contributions are of independent interest. For instance, our analysis of the meta SRM algorithms is based on a general reduction from cumulative regret minimization in \cref{sec:Cum2Sim}, which yields novel and easily implementable algorithms for Bayesian and frequentist SRM, based on \emph{Thompson sampling (TS)} and \emph{upper confidence bounds (UCBs)} \citep{lu2019information}. To the best of our knowledge, only \citet{komiyama21optimal} studied Bayesian SRM before (\cref{sec:related-work}). In \cref{sec:LinGaus}, we also extend the analysis of frequentist meta-learning in \citet{simchowitz2021bayesian} to structured bandit problems.

\section{Problem Setup}\label{sec:Setup}

In meta SRM, we consider $m$ bandit problems with arm set $\ca{A}$ that appear sequentially and each is played for $n$ rounds. At the beginning of each task (bandit problem) $s \in [m]$, the mean rewards of its arms $\mu_s \in \mathbb{R}^{\ca{A}}$ are sampled i.i.d.\ from a prior distribution $P_*$. We define $[m]=\{1,2,\cdots,m\}$ for any integer $m$. We apply a base SRM algorithm, \alg, to task $s$ and denote this instance by $\Alg_s$. The algorithm interacts with task $s$ for $n$ rounds. In round $t \in [n]$ of task $s$, $\Alg_s$ pulls an arm $A_{s,t}\in\A$ and observes its reward $Y_{s,t}(A_{s,t})$, where $\E[Y_{s,t}(a)]=\mu_s(a)$. We assume that $Y_{s,t}(a)\sim \nu(a;\mu_s)$ where $\nu(\cdot;\mu_s)$ is the reward distribution of all arms with parameter (mean) $\mu_s$. After the $n$ rounds the algorithm returns arm $\Ahat_{\Alg_s}$ or simply $\Ahat_s$ as the {\em best arm}. Let $A^*_s=\argmax_{a\in\A}\mu_s(a)$ be the best arm in task $s$. We define the \emph{per-task simple regret} for task $s$ as
\begin{equation}
\label{eq:per-task-SR}
\SR_s(n, P_*) = {\E}_{\mu_s\sim P_*}\E_{\mu_s}[\Delta_{s}],    
\end{equation}
where $\Delta_{s}=\mu_s(A^*_s)-\mu_s(\Ahat_s)$. 
The outer expectation is w.r.t.~the randomness of the task instance, and the inner one is w.r.t.~the randomness of rewards and algorithm. This is the common frequentist simple regret averaged over instances drawn from $P_*$.

In the {\em frequentist} setting, we assume that $P_*$ is unknown but fixed, and define the \emph{frequentist meta simple regret} as
\begin{equation}
\label{eq:frequentist-meta-SR}
    \SR(m,n,P_*)=\sum_{s=1}^m\SR_s(n, P_*)\;.
\end{equation}
In the {\em Bayesian} setting, we still assume that $P_*$ is unknown. However, we know that it is sampled from a known \emph{meta prior} $Q$. We define \emph{Bayesian meta simple regret} as
\begin{equation}
\label{eq:Bayesian-meta-SR}
    \BSR(m,n)= {\E}_{P_*\sim Q}[\SR(m,n,P_*)]\;.
\end{equation}

\section{Bayesian Meta-SRM}\label{sec:BMBAI}

In this section, we present our Bayesian meta SRM algorithm (\bMetaSRM), whose pseudo-code is in \cref{alg:B-MBAI}. The key idea is to deploy \alg for each task with an adaptively refined prior learned from the past interactions, which we call an \emph{uncertainty-adjusted prior}, $P_s(\mu)$. This is an approximation to $P_*$ and it is the posterior density of $\mu_s$ given the history up to task $s$. At the beginning of task $s$, \bMetaSRM instantiates \alg with $P_s$, denoted as $\alg_s=\alg(P_s)$, and uses it to solve task $s$. 

The base algorithm \alg is \emph{Thompson Sampling (TS)} or \emph{Bayesian UCB (\ucb)} \citep{lu2019information}. During its execution, $\alg_s$ keeps updating its posterior over $\mu_s$ as $P_{s,t}(\mu_s)\propto \ca{L}_{s,t}(\mu_s)P_s(\mu_s)$, where $\ca{L}_{s,t}(\mu_s)=\prod_{\ell=1}^t \P(Y_{s,\ell}|A_{s,\ell},\mu_s)$ is the likelihood of observations in task $s$ up to round $t$ under task parameter $\mu_s$. TS pulls the arms proportionally to being the best w.r.t.~the posterior. More precisely, it samples $\tilde{\mu}_{s,t}\sim P_{s,t}$ and then pulls arm $A_{s,t}\in\argmax_{a\in\A}\tilde{\mu}_{s,t}(a)$. \ucb is the same but it pulls the arm with largest Bayesian upper confidence bound (see \cref{app:BayesReg-pfs} and \cref{eq:Ucb} for details).

The critical step is how $P_s$ is updated. Let $\theta_*$ be the parameter of $P_*$~. At task $s$, \bMetaSRM maintains a posterior density over the parameter $\theta_*$, called \emph{meta-posterior} $Q_s(\theta)$, and uses it to compute $P_s(\mu)$.
We use the following recursive rule from Proposition 1 of \citet{basu2021no} to update $Q_s$ and $P_s$. 
\begin{proposition}\label{prop:posterior}
Let $\ca{L}_{s-1}(\cdot)=\ca{L}_{s-1,n}(\cdot)$ be the likelihood of observations right before the start of task $s$. We let $P_\theta$ be the prior distribution parameterized by $\theta$. Then \alg computes $Q_s$ and $P_s$ as
\begin{align}
     Q_s(\theta)&=\int_{\mu}\ca{L}_{s-1}(\mu)P_{\theta}(\mu)\da\kappa_2(\mu)Q_{s-1}(\theta)\label{eq:Qs},\quad & \forall \theta\\
     P_s(\mu) &= \int_{\theta}P_{\theta}(\mu) Q_s(\theta)\da\kappa_1(\theta), \quad & \forall \mu\label{eq:Ps}
\end{align}
where $\kappa_1$ and $\kappa_2$ are the probability measures of $\theta$ and $\mu$. \todob{The role of $P_\theta$ is unclear based on this description.} We initialize \cref{eq:Qs} with $\ca{L}_0=1$ and $Q_0=Q$, where $Q$ is the meta prior.
\end{proposition}
Note that this update rule is computationally efficient for Gaussian prior with Gaussian meta-prior, but not many other distributions. This computational issue can limit the applicability of our Bayesian algorithm.

When task $s$ ends, $\alg_s$ returns the best arm $\Ahat_{\Alg_s}$ by sampling from the distribution 
\begin{equation}
\label{eq:best-arm-BMetaSRM}
\Ahat_{\Alg_s}\sim\rho_s, \qquad \rho_s(a):=\frac{N_{a,s}}{n},
\end{equation}
where $N_{a,s}:=|\{t \in [n]:A_{s,t}=a\}|$ is the number of rounds where arm $a$ is pulled. That is, the algorithm chooses the arms proportionally to their number of pulls. This decision rule facilitates the analysis of our algorithms based on a reduction from cumulative to simple regret. We develop this reduction in \cref{sec:Cum2Sim} and show that per-task simple regret is essentially the cumulative regret divided by $n$. This yields novel algorithms for Bayesian and frequentist SRM with guarantees.

\subsection{Cumulative to Simple Regret Reduction}\label{sec:Cum2Sim}

Fix task $s$ and consider an algorithm that pulls a sequence of arms $(A_{s,t})_{t\in[n]}$. Let its per-task cumulative regret with prior $P$ be
$$R_s(n, P):=\E_{\mu_s\sim P}\E_{\mu_s}\left[n\mu_s(A^*_s) - \sum_{t=1}^n\mu_s(A_{s,t})\right]\;,$$
where the inner expectation is taken over the randomness in the rewards and algorithm. Now suppose that at the end of the task, we choose arm $a$ with probability $\rho_s(a)$ and declare it to be the best arm $\Ahat_s$. Then the per-task simple regret of this procedure is bounded as follows.
\begin{restatable}[Cumulative to Simple Regret]{proposition}{cumTosim}
\label{prop:cum2sim}
For task $s$ with $n$ rounds, if we return an arm with probability proportional to its number of pulls as the best arm, the per-task simple regret with prior $P$ is $\SR_s(n, P) = R_s(n, P)/n$.
\end{restatable}
We prove this proposition in \cref{app:Cum2Sim} using the linearity of expectation and properties of $\rho_s$. Note that \cref{prop:cum2sim} applies to both frequentist and Bayesian \emph{meta} simple regret. This is because the former is a summation of $\SR_s$ over tasks, and the latter is achieved by taking an expectation of the former over $P_*$. 

\setlength{\textfloatsep}{5pt}%
\begin{algorithm}[tb]
   \caption{Bayesian Meta-SRM (\bMetaSRM)}
    \label{alg:B-MBAI}
\begin{algorithmic}

    \STATE {\bfseries Input:} Meta prior $Q$, base algorithm \alg
    \STATE {\bf Initialize:} Meta posterior $Q_0\gets Q$
    \FOR{$s=1,\dots,m$}{
    \STATE Receive the current task $s,\;\;\mu_s\sim P_*$ 
    \STATE Compute meta posterior $Q_s$ using \cref{eq:Qs}
    \STATE Compute uncertainty-adjusted prior $P_s$ using \cref{eq:Ps}
    \STATE Instantiate \alg for task $s,\;\;\Alg_s\gets\alg(P_s)$
    \STATE Run $\Alg_s$ for $n$ rounds 
    \STATE Return the best arm $\Ahat_{\Alg_s}\sim\rho_s$ using~\cref{eq:best-arm-BMetaSRM}
    
    }
    \ENDFOR
\end{algorithmic}
\end{algorithm}

\subsection{Bayesian Regret Analysis}
\label{sec:BayesReg}

Our analysis of \bMetaSRM is based on results in~\citet{basu2021no} and~\citet{lu2019information}, combined with \cref{sec:Cum2Sim}. Specifically, let $\Gamma_{s,t}$ be an information-theoretic constant independent of $m$ and $n$ that bounds the instant regret of the algorithm at round $t$ of task $s$. We defer its precise definition to \cref{app:BayesReg-pfs} as it is only used in the proofs. The following generic bound for the Bayesian meta simple regret of \bMetaSRM holds.

\begin{restatable}[Information Theoretic Bayesian Bound]{theorem}{infoBound}
\label{lem:infoBd}
Let $\{\Gamma_s\}_{s\in[m]}$ and $\Gamma$ be non-negative constants, such that $\Gamma_{s,t}\leq \Gamma_s\leq \Gamma$ holds for all $s\in[m]$ and $t\in[n]$ almost surely. Then, the Bayesian meta simple regret (Eq.~\ref{eq:Bayesian-meta-SR}) of \bMetaSRM satisfies
\begin{align}
\label{eq:Bayes-meta-SR-bound}
    \BSR(m,n)& \leq \Gamma\sqrt{\frac{m}{n}\;\Info(\theta_*;\tau_{1:m})}
    \\&\!\!\!\!\!\!\!\!+\sum_{s=1}^m\Gamma_s\sqrt{\frac{\Info(\mu_s;\tau_s|\theta_*,\tau_{1:s-1})}{n}}
    +\sum_{s=1}^m\sum_{t=1}^n\frac{\E[\beta_{s,t}]}{n},\nonumber
\end{align}
where $\tau_{1:s}=\oplus_{\ell=1}^s (A_{\ell,1}, Y_{\ell,1},\cdots, A_{\ell,n},Y_{\ell,n})$ is the trajectory up to task $s$, $\tau_s$ is similarly defined for the history only in task $s$, and 
$\Info(\cdot;\cdot)$ and $\Info(\cdot;\cdot|\cdot)$ are mutual information and conditional mutual information, respectively.
\end{restatable}
The proof is in \cref{app:BayesReg-pfs}. It builds on the analysis in \citet{basu2021no} and uses our reduction in \cref{sec:Cum2Sim}. Our reduction readily applies to Bayesian meta simple regret by linearity of expectation.

The first term in \cref{eq:Bayes-meta-SR-bound} is the price for learning the prior parameter $\theta_*$ and the second one is the price for learning the mean rewards of tasks $(\mu_s)_{s \in [m]}$ given known $\theta_*$. It has been shown in many settings that the mutual information terms grow slowly with $m$ and $n$ \citep{lu2019information,basu2021no}, and thus the first term is $\Tilde{O}(\sqrt{m/n})$ and negligible. The second term is $\Tilde{O}(m/\sqrt{n})$, since we solve $m$ independent problems, each with $\tilde{O}(1 / \sqrt{n})$ simple regret. In \cref{sec:LinGaus}, we discuss a bandit environment where $\Gamma_{s,t}$ and $\beta_{s,t}$ are such that the last term of the bound is comparable to the rest. This holds in several other environments discussed in \citet{lu2019information,basu2021no}, and~\citet{liu2022gaussian}.

\section{Frequentist Meta-SRM}\label{sec:fMBAI}

In this section, we present our frequentist meta SRM algorithm (\fMetaSRM), whose pseudo-code is in \cref{alg:f-MBAI}. Similarly to \bMetaSRM, \fMetaSRM uses TS or UCB as its base algorithm \alg. However, it directly estimates its prior parameter, instead of maintaining a meta-posterior. At the beginning of task $s\in[m]$, \fMetaSRM explores the arms for a number of rounds using an \emph{exploration strategy} denoted as \expl. This strategy depends on the problem class and we specify it for two classes in \cref{sec:examples}. \fMetaSRM uses samples collected in the exploration phase of all the tasks up to task $s$, $\tilde{\tau}_s$, to update its estimate of the prior parameter $\hat{\theta}_s$. Then, it instantiates the base algorithm with this estimate, denoted as $\Alg_s=$\alg$(\hat{\theta}_s)$, and uses $\Alg_s$ for the rest of the rounds of task $s$. Here $\alg(\theta):=\alg(P_{\theta})$ is the base algorithm \alg instantiated with prior parameter $\theta$ (Note that we used a slightly different parameterization of $\alg$ compared to \cref{sec:BMBAI}). When task $s$ ends, $\Alg_s$ returns the best arm $\Ahat_{\Alg_s}$ by sampling from the probability distribution $\rho_s$ defined in \cref{eq:best-arm-BMetaSRM}. 

While \bMetaSRM uses a Bayesian posterior to maintain its estimate of $\theta_*$, \fMetaSRM relies on a frequentist approach. Therefore, it applies to settings where computing the posterior is not computationally feasible. Moreover, we can analyze \fMetaSRM for general settings beyond Gaussian bandits.

\begin{algorithm}[tb]
   \caption{Frequentist Meta-SRM (\fMetaSRM)}
    \label{alg:f-MBAI}
    \begin{algorithmic}
    \STATE {\bfseries Input:} Exploration strategy \expl, base algorithm \alg
    \STATE {\bfseries Initialize:} $\tilde{\tau}_1\gets\emptyset$
    \FOR{$s=1,\dots,m$}{
    \STATE Receive the current task $s,\;\;\mu_s\sim P_*$
    \STATE Explore the arms using \expl 
    \STATE Append explored arms and their observations to $\tilde{\tau}_s$
    \STATE Compute $\hat{\theta}_s$ using $\tilde{\tau}_s$ as an estimate of $\theta_*$ 
    \STATE Instantiate \alg for task $s,\;\;\Alg_s \gets \alg(\hat{\theta}_s)$
    \STATE Run $\Alg_s$ for the rest of the $n$ rounds
    \STATE Return the best arm $\Ahat_{\Alg_s}\sim\rho_s$ using~\cref{eq:best-arm-BMetaSRM}
    \STATE $\tilde{\tau}_{s+1} \gets \tilde{\tau}_s$
    }
    \ENDFOR
\end{algorithmic}
\end{algorithm}

\subsection{Frequentist Regret Analysis}\label{sec:frqAny}

In this section, we prove an upper bound for the frequentist meta simple regret (Eq.~\ref{eq:frequentist-meta-SR}) of \fMetaSRM with TS \alg. To start, we bound the per-task simple regret of \alg relative to {\em oracle} that knows $\theta_*$. To be more precise, this is the difference between the means of arms returned by \alg instantiated with some prior parameter $\theta$ and the true prior parameter $\theta_*$.

The \emph{total variation} (TV) distance for two distributions $P$ and $P'$ over the same probability space $(\Omega,\ca{F})$\footnote{$\Omega$ is the sample space and $\ca{F}$ is the sigma-algebra.} is defined as
$
    \tv(P\parr P') := \sup_{E\in\ca{F}}\;|P(E)-P'(E)|
    .
    $
We use TV to measure the distance between the estimated and true priors. We fix task $s$ and drop subindexing by $s$. In the following, we bound the per-task simple regret of $\alg(\theta)$ relative to oracle $\alg(\theta_*)$. 

\begin{restatable}{theorem}{thmfrqMSR}
\label{thm:frqMSR}%
Suppose $P_{\theta_*}$ is the true prior
of the tasks and satisfies $P_{\theta_*}(\diam(\mu)\leq B)=1$, where $\diam(\mu):=\sup_{a\in\A}\mu(a)-\inf_{a\in\A}\mu(a)$. Let $\theta$ be a prior parameter, such that $\tv(P_{\theta_*}\parr P_{\theta})=\epsilon$. Also, let $\Ahat_{\alg(\theta_*)}$ and $\Ahat_{\alg(\theta)}$ be the arms returned by $\alg(\theta_*)$ and $\alg(\theta)$, respectively. Then we have
\begin{align}
\E_{\mu_\sim P_{\theta_*}}\E\big[\mu(\Ahat_{\alg(\theta_*)})-\mu(\Ahat_{\alg(\theta)})\big]\leq 2n\epsilon B.
\label{eq:BB}
\end{align}
Moreover, if the prior is coordinate-wise $\sigma_0^2$-sub-Gaussian (\cref{defi:Tails} in \cref{sec:frqPfs}), then we may write the RHS of \cref{eq:BB} as
$
2n\epsilon 
\Big(\diam\big(\E_{\theta_*}[\mu]\big)+
\sigma_0\Big(8+5\sqrt{\log\frac{|\A|}{\min(1,2n \epsilon)}}\Big)\Big),
$
where $\E_{\theta_*}[\mu]$ is the expectation of the mean reward of the arms, $\mu$, given the true prior $\theta_*$.%
\end{restatable}

The proof in \cref{sec:frqPfs} uses the fact that TS is a $1$-Monte Carlo algorithm, as defined by~\citet{simchowitz2021bayesian}. It builds on \citet{simchowitz2021bayesian} analysis of the cumulative regret, and extends it to simple regret. We again use our reduction in \cref{sec:Cum2Sim}, which shows how it can be applied to a frequentist setting.

\cref{thm:frqMSR} shows that an $\epsilon$ prior misspecification leads to $O(n \epsilon)$ simple regret cost in \fMetaSRM. The constant terms in the bounds depend on the prior distribution. In particular, for a bounded prior, they reflect the variability (diameter) of the expected mean reward of the arms. Moreover, under a sub-Gaussian prior, the bound depends logarithmically on the number of arms $|\A|$ and sub-linearly on the prior variance proxy $\sigma_0^2$.

Next, we bound the frequentist meta simple regret (Eq.~\ref{eq:frequentist-meta-SR}) of \fMetaSRM.

\begin{restatable}[Meta Simple Regret of \fMetaSRM]{cor}{corfrqSR}\label{cor:frqSR}
Let the \expl strategy in \cref{alg:f-MBAI} be such that $\epsilon_s=\tv(P_{\theta_*}\parr P_{\hat{\theta}_s})=O(1/\sqrt{s})$ for each task $s\in[m]$. Then the frequentist meta simple regret of \fMetaSRM is bounded as %
\begin{align}
\label{eq:fMeta-SRM-bound}
\SR(m,n,P_{\theta_*})=O\left(2\sqrt{m} nB + m\sqrt{|\A|/n}\right).
\end{align}
\end{restatable} 

The proof is in \cref{sec:frqPfs} and decomposes the frequentist meta simple regret into two terms: (i) the per-task simple regret of $\alg(\hat{\theta}_s)$ relative to oracle $\alg(\theta_*)$ in task $s$, which we bound in \cref{thm:frqMSR}, and (ii) the meta simple regret of the oracle $\alg(\theta_*)$, which we bound using our cumulative regret to simple regret reduction (\cref{sec:Cum2Sim}).

The $O(\sqrt{m}n)$ term is the price of estimating the prior parameter, because it is the per-task simple regret relative to the oracle. The $O(m\sqrt{|\A|/n})$ term is the meta simple regret of the oracle over $m$ tasks.

Comparing to our bound in \cref{lem:infoBd}, \bMetaSRM has a lower regret of $O(\sqrt{m/n}+m/\sqrt{n})=O(m/\sqrt{n})$. 
More precisely, only the price for learning the prior is different as both bounds have $O(m / \sqrt{n})$ terms. Note that despite its smaller regret bound, \bMetaSRM may not be computationally feasible for arbitrary distributions and priors, while \fMetaSRM is since it directly estimates the prior parameter using frequentist techniques.

\subsection{Lower Bound}\label{sec:fMBAI-LB}

In this section, we prove a lower bound on the relative per-task simple regret of a $\gamma$-shot TS algorithm, i.e., a TS algorithm that takes $\gamma\in\mathbb{N}$ samples (instead of $1$) from the posterior in each round. This lower bound compliments our upper bound in \cref{thm:frqMSR} and shows that \cref{eq:BB} is near-optimal. The proof of our lower bound builds on a cumulative regret lower bound in Theorem 3.3 of \citet{simchowitz2021bayesian} and extends it to simple regret. We present the proof in \cref{app:LB}.

\begin{theorem}[Lower Bound]
\label{thm:lower bound}
Let $\TS_{\gamma}(\theta)$ be a $\gamma$-shot TS algorithm instantiated with the prior parameter $\theta$. Also let $P_\theta$ and $P_{\theta'}$ be two task priors. Let $\mu \in [0, 1]^{\A}$ and fix a tolerance $\eta\in(0,\frac{1}{4})$. Then there exists a universal constant $c_0$ such that for any horizon $n \geq \frac{c_0}{\eta}$, number of arms $|\A| = n\lceil\frac{c_0}{\eta}\rceil$, and error $\epsilon\leq \frac{\eta}{c_0\gamma n}$, we have $\tv(P_{\theta}\parr P_{\theta'}) =\epsilon$ and the difference of per-task simple regret of $\TS_{\gamma}(\theta)$ and $\TS_{\gamma}(\theta')$ satisfies $\E[\mu(\Ahat_{\TS_{\gamma}(\theta)})] - \E[\mu(\Ahat_{\TS_{\gamma}(\theta')})] \geq (\frac{1}{2}-\eta) \gamma n\epsilon$.
\end{theorem}
This lower bound holds for any setting with large enough $n$ and $|\A|=O(n^2)$, and a small prior misspecification error $\epsilon=O(1/n^2)$. This makes it relatively general.

\section{Meta-Learning Examples}\label{sec:examples}

In this section, we apply our algorithms to specific priors and reward distributions. The main two are the Bernoulli and linear (contextual) Gaussian bandits. We analyze \fMetaSRM in an \emph{explore-then-commit} fashion, where \fMetaSRM estimates the prior using \expl in the first $m_0$ tasks and then commits to it. This is without loss of generality and only for simplicity.

\subsection{Bernoulli Bandits}\label{sec:Examp-Bern}

We start with a Bernoulli multi-armed bandit (MAB) problem, as TS was first analyzed in this setting \citep{agrawal2012analysis}. Consider Bernoulli rewards with beta priors for $\A=[K]$ arms. In particular, assume that the prior is $P_{*}=\bigotimes_{a\in\A}\text{Beta}(\alpha^*_a,\beta^*_a)$. 
Therefore, $\alpha^*_a$ and $\beta^*_a$ are the prior parameters of arm $a$ and the arm mean $\mu_s(a)$ is the probability of getting reward 1 for arm $a$ when it is pulled. $\text{Beta}(\alpha,\beta)$ is the beta distribution 
with a support on $(0, 1)$ with parameters $\alpha>0$ and $\beta>0$. 

\bMetaSRM in this setting does not have a computationally tractable meta-prior \citep{basu2021no}. We can address this in practice by discretization and using TS as described in Section 3.4 of \citet{basu2021no}. However, the theoretical analysis for this case does not exist. This is because a computationally tractable prior for a product of beta distributions does not exist. 
It is challenging to generalize our Bayesian approach to this class of distributions as we require more than the standard notion of conjugacy. 

In the contrary, \fMetaSRM directly estimates the beta prior parameters, $(\alpha^*_a)_{a \in \A}$ and $(\beta^*_a)_{a \in \A}$ based on the observed Bernoulli rewards as follows. The algorithm explores only in $m_0\leq m$ tasks. \expl samples arm 1 in the first $t_0$ rounds of first $m_0/K$ tasks, and arm 2 in the next $m_0/K$ tasks similarly, and so on for arm 3 to $K$. In other words, \expl samples arm $a\in[K]$ in the first $t_0$ rounds of $a$'th batch of size $m_0/K$ tasks. Let $X_s$ denote the cumulative reward collected in the first $t_0$ rounds of task $s$. Then, the random variables $X_1, \ldots , X_{m_0/K}$ are i.i.d.\ draws from a Beta-Binomial distribution (BBD) with parameters $(\alpha^*_1, \beta^*_1, t_0)$, where $t_0$ denotes the number of trials of the binomial component. Similarly, $X_{(m_0/K)+1},\cdots,X_{2m_0/K}$ are i.i.d.\ draws from a BBD with parameters $(\alpha^*_2, \beta^*_2, t_0)$. In general, $X_{(a-1)(m_0/K)+1},\cdots,X_{a m_0/K}$ are i.i.d.\ draws from a BBD with parameters $(\alpha^*_a, \beta^*_a, t_0)$. Knowing this, it is easy to calculate the prior parameters for each arm using the method of moments \citep{tripathi1994estimation}. The detailed calculations are in \cref{app:MoM-Ber}. We prove the following result in \cref{sec:freBern-pfs}. 

\begin{restatable}[Frequentist Meta Simple Regret, Bernoulli]{cor}{corBernSR}
\label{cor:frqSR-Bern}
Let $\alg$ be a TS algorithm that uses the method of moments described and detailed in \cref{app:MoM-Ber}, to estimate the prior parameters with $\MBern$ exploration tasks (explore-then-commit). Then the frequentist meta simple regret of \fMetaSRM satisfies $\SR(m,n,P_{\theta_*})=O\big(2 m n \epsilon+ m\sqrt{\frac{ |\A|\log(n)}{n}} + m_0 \big),$ for $m\geq m_0$ with probability at least $1-\delta$.
\end{restatable} 
With small enough $\epsilon$, the bound shows $\Tilde{O}(m/\sqrt{|\A|/n})$ scaling which we conjecture is the best an oracle that knows the correct prior of each task could do in expectation.
The bound seems to be only sublinear in $n$ if $\epsilon = O(1 / n^{3 / 2})$.
However, since $\epsilon\propto m_0^{-1/2}$ and we know $\sum_{z=1}^m z^{-1/2}=m^{1/2}$, if the exploration continues in all tasks, the regret bound above simplifies to $O\Big( \sqrt{m} n+ m\sqrt{\frac{ |\A|\log(n)}{n}} \Big)$.

\subsection{Linear Gaussian Bandits}\label{sec:LinGaus}

In this section, we consider linear contextual bandits. 
Suppose that each arm $a\in\A$ is a vector in $\mathbb{R}^d$ and $|\A|=K$. Also, assume $\nu_s(a;\mu_s)=\Ngaus(a\T\mu_s,\sigma^2)$, i.e., with a little abuse of notation $\mu_s(a)=a\T\mu_s$, where $\mu_s$ is the parameter of our linear model. A conjugate prior for this problem class is $P_* = \Ngaus(\theta_*, \Sigma_0)$, where $\Sigma_0\in \mathbb{R}^{d\times d}$ is known and we learn $\theta_*\in \mathbb{R}^d$.

In the Bayesian setting, we assume that the meta-prior is $Q=\Ngaus(\psi_q,\Sigma_q)$, where $\psi_q\in\mathbb{R}^d$ and $\Sigma_q\in\mathbb{R}^{d\times d}$ are both known. In this case, the meta-posterior is $Q_s=\Ngaus(\hat{\theta}_s,\hat{\Sigma}_s)$, where $\hat{\theta}_s\in\R^d$ and $\hat{\Sigma}_s\in\R^{d\times d}$ are calculated as
\begin{align*}
  \hat{\theta}_s
  & = \hat{\Sigma}_s \Big(\Sigma_q^{-1} \psi_q +
  \sum_{\ell = 1}^{s - 1} \frac{B_\ell}{\sigma^2} - \frac{V_\ell}{\sigma^2}
  \Big(\Sigma_0^{-1} + \frac{V_\ell}{\sigma^2}\Big)^{-1} \frac{B_\ell}{\sigma^2}\Big), \\
  \hat{\Sigma}_s^{-1}
  & = \Sigma_q^{-1} +
  \sum_{\ell = 1}^{s - 1} \frac{V_\ell}{\sigma^2} - \frac{V_\ell}{\sigma^2}
  \Big(\Sigma_0^{-1} + \frac{V_\ell}{\sigma^2}\Big)^{-1} \frac{V_\ell}{\sigma^2},
\end{align*}
where $V_\ell = \sum_{t = 1}^n A_{\ell, t} A_{\ell, t}\T$ is the outer product of the feature vectors of the pulled arms in task $\ell$ and $B_\ell = \sum_{t = 1}^n A_{\ell, t} Y_{\ell, t}(A_{\ell, t})$ is their sum weighted by their rewards (see Lemma 7 of \citet{kveton2021metathompson} for more details).
By \cref{prop:posterior}, we can calculate the task prior for task $s$ as $P_s = \Ngaus( \hat{\theta}_s, \hat{\Sigma}_s + \Sigma_0)$. When $K = d$ and $\A$ is the standard Euclidean basis of $\R^d$, the linear bandit reduces to a $K$-armed bandit.

Assuming that $\max_{a\in\A}\vnormf{a}\leq 1$ by a scaling argument, the following result holds by an application of our reduction in \cref{sec:Cum2Sim}, 
and we prove it in \cref{app:LinBan-Bay-pfs}. For a matrix $A\in\R^{d\times d}$, let $\lambda_{1}(A)$ denote its largest eigenvalue.

\begin{restatable}[Bayesian Meta Simple Regret, Linear Bandits]{cor}{corLinBanBayes}
\label{cor:linBan-Bayes}
    For any $\delta\in(0,1]$, the Bayesian meta simple regret of \bMetaSRM in the setting of \cref{sec:LinGaus} with TS \alg is bounded as
    $
       \BSR(m,n)\leq c_1\sqrt{d m /n}+ (m+c_2)\SR_{\delta}(n) + c_3 d m/n,
     $
    where $c_1=O(\sqrt{\log(K/\delta)\log m})$, $c_2=O(\log m)$, and $c_3$ is a constant in $m$ and $n$. Also $\SR_{\delta}(n)$ is the per-task simple regret bounded as $\SR_{\delta}(n)\leq c_4\sqrt{\tfrac{d}{n}}+\sqrt{2\delta \lambda_1(\Sigma_0)}$, where $c_4=O\Big(\sqrt{\log(\tfrac{K}{\delta})\log n}\Big)$.
\end{restatable}

The first term in the regret is $\tilde{O}(\sqrt{d m/n})$ and represents the price of learning $\theta_*$. The second term is the simple regret of $m$ tasks when $\theta_*$ is known and is $\tilde{O}(m\sqrt{d/ n})$. The last term is the price of the forced exploration and is negligible, $\tilde{O}(m/n)$. Comparing to the analysis in \citet{basu2021no}, we prove a similar bound for \bMetaSRM with \ucb base algorithm in \cref{sec:ucb-proofs}.

In the frequentist setting, we simplify the setting to $P_*= \Ngaus(\theta_*, \sigma_0^2 I_d)$. 
The case of general covariance matrix for the MAB Gaussian is dealt with in \citet{simchowitz2021bayesian}. We extend the results of \citet{simchowitz2021bayesian} for meta-learning to linear bandits.
Our estimator of $\theta_*$, namely $\hat{\theta}_s$, is such that $\tv\left(P_{\theta_s}\parr P_{\hat{\theta}_*}\right)$ is bounded based on all the observations up to task $s$. We show that for any $\epsilon, \delta \in (0, 1)$, with probability at least $1 - \delta$ over the realizations of the tasks and internal
randomization of the meta-learner, $\hat{\theta}_*$ is close to $\theta_*$ in TV distance.

The key idea of the analysis is bounding the regret relative to an oracle. We use \cref{thm:frqMSR} to bound the regret of \fMetaSRM relative to an oracle $\alg(\theta_*)$ which knows the correct prior. Our analysis and estimator also apply to sub-Gaussian distributions, but we stick to linear Gaussian bandits for readability. Without loss of generality, let $a_1,\ldots,a_d$ be a basis for $\A$ such that $\Spn(\{a_1,\ldots,a_d\}) = \R^d$. Resembling \cref{sec:Examp-Bern}, we only need to explore the basis. The exploration strategy, \expl in \cref{alg:f-MBAI}, samples the basis $a_1, \ldots, a_d$ in the first $m_0\leq m$ tasks. Then the least-squares estimate of $\theta_*$ is
\begin{align}
    \hatTheta\;, \label{eq:LinGaus-theta-hat}
\end{align}
where $\VmDef$ is the outer product of the basis. This gives an unbiased estimate of $\theta_*$. Then we can guarantee the performance of \expl as follows.
\begin{restatable}[Linear Bandits Frequentist Estimator]{theorem}{thmlinBanf}
\label{thm:linBan-f-est}
    In the setting of \cref{sec:LinGaus}, for any $\epsilon$ and $\delta\in(2e^{-d},1)$, if $n\geq d$ and
    $
        \MLin
        ,
    $
then $\tv(P_{\theta_*}\parr P_{\hat{\theta}_*})\leq \epsilon$ with probability at least $1-\delta$.
\end{restatable}
We prove this by bounding the TV distance of the estimate and correct prior using the Pinsker's inequality. Then the KL-divergence of the correct prior and the prior with parameter $\hat{\theta}_*$ boils down to $\vnormf{\theta_*-\hat{\theta}_*}_2$, which is bounded by the Bernstein's inequality (see \cref{app:fLinGausPfs} for the proof).

Now it is easy to bound the frequentist meta simple regret of \fMetaSRM using the sub-Gaussian version of \cref{cor:frqSR} in \cref{sec:frqPfs}. We prove the following result in \cref{app:fLinGausPfs} by decomposing the simple regret into the relative regret of the base algorithm w.r.t.~the oracle.

\begin{restatable}[Frequentist Meta Simple Regret, Linear Bandits]{cor}{corfLinSR}
\label{cor:frqSR-Lin}
In \cref{alg:f-MBAI}, let $\alg$ be a TS algorithm and use \cref{eq:LinGaus-theta-hat} for estimating the prior parameters with ${m_0}^3\geq \left(\frac{  d \log(2/\sqrt{\delta}) \sum_{i=1}^d\sigma_i^2 }{2\sigma_0\lambda_d^{4}(\sum_{i=1}^d a_i a_i\T) \epsilon^4}\right)$. Then the frequentist meta simple regret of \cref{alg:f-MBAI}
    is
    $\tilde{O}\left(2 m^{1/4} n ~\diam(\E_{\theta_*}[\mu]) + m \frac{d^{3/2}\log K}{\sqrt{n}}\right)$
    with probability at least $1-\delta$.
\end{restatable}
This bound is $\tilde{O}(m^{1/4} n \|\theta_*\|_{\infty} + m d^{3/2}/\sqrt{n})$, where $\|\cdot\|_{\infty}$ is the infinity norm. The first term is the price of estimating the prior and the second one is the standard frequentist regret of linear TS for $m$ tasks divided by $n$, $\tilde{O}(m d^{3/2}/\sqrt{n})$. Compared to \cref{cor:linBan-Bayes}, the above regret bound is looser.

\section{Related Work}
\label{sec:related-work}

To the best of our knowledge, there is no prior work on meta-learning for SRM. We build on several recent works on meta-learning for cumulative regret minimization \citep{bastani19meta,cella20metalearning,kveton2021metathompson,basu2021no,simchowitz2021bayesian}. Broadly speaking, these works either study a Bayesian setting \citep{kveton2021metathompson,basu2021no,hong2022hierarchical}, where the learning agent has access to a prior distribution over the meta-parameters of the unknown prior $P_*$; or a frequentist setting \citep{bastani19meta,cella20metalearning,simchowitz2021bayesian}, where the meta-parameters of $P_*$ are estimated using frequentist estimators. We study both the Bayesian and frequentist settings. Our findings are similar to prior works, that the Bayesian methods have provably lower regret but are also less general when insisting on the exact implementation.

Meta-learning is an established field of machine learning \citep{thrun96explanationbased,thrun98lifelong,baxter98theoretical,baxter00model,finn18probabilistic}, and also has a long history in multi-armed bandits \citep{azar13sequential,gentile14online,deshmukh17multitask}. Tuning of bandit algorithms is known to reduce regret \citep{vermorel05multiarmed,maes12metalearning,kuleshov14algorithms,hsu19empirical} and can be viewed as meta-learning. However, it lacks theory. Several papers tried to learn a bandit algorithm using policy gradients \citep{duan16rl2,boutilier20differentiable,kveton20differentiable,yang20differentiable,min20policy}. These works focus on offline optimization against a known prior $P_*$ and are in the cumulative regret setting.

Our SRM setting is also related to fixed-budget \emph{best-arm identification (BAI)} \citep{Gabillon-2012,pmlr-v139-alieva21a,azizi2021fixedbudget}. In BAI, the goal is to control the probability of choosing a suboptimal arm. The two objectives are related because the simple regret can be bounded by the probability of choosing a suboptimal arm multiplied by the maximum gap.

While SRM has a long history \citep{audibert-2010-BAI,kaufmann16complexity}, prior works on Bayesian SRM are limited. \citet{russo2020simple} proposed a TS algorithm for BAI. However, its analysis and regret bound are frequentist. The first work on Bayesian SRM is \citet{komiyama21optimal}. Beyond establishing a lower bound, they proposed a Bayesian algorithm that minimizes the (Bayesian) per-task simple regret in \cref{eq:per-task-SR}. This algorithm does not use the prior $P_*$ and is conservative. As a side contribution of our work, we establish Bayesian per-task simple regret bounds for posterior-based algorithms in this setting.

\begin{figure*}[tb]
    \centering
    \begin{minipage}{.45\textwidth}
        \includegraphics[width=\textwidth]{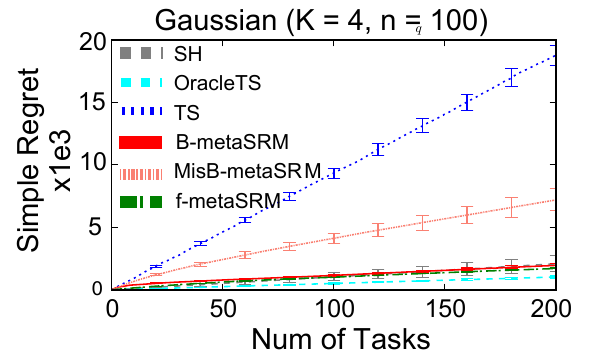}
    \end{minipage}
    \hfill
    \begin{minipage}{.45\textwidth}
        \includegraphics[width=\textwidth]{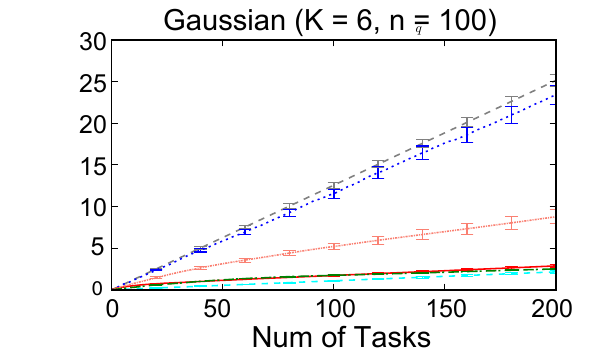}
    \end{minipage}
    \hfill
    \begin{minipage}{.45\textwidth}
        \includegraphics[width=\textwidth]{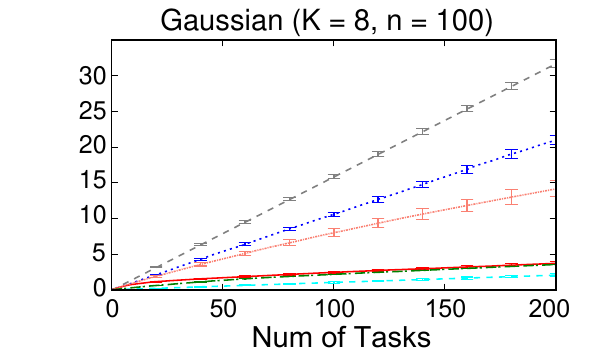}
    \end{minipage}
    \caption{Learning curves for Gaussian MAB experiments. The error bars are standard deviations from 100 runs.}
    \label{fig:Gaus}
\end{figure*}

\begin{figure*}[tb]
    \centering
    \begin{minipage}{.45\textwidth}
        \includegraphics[width=\textwidth]{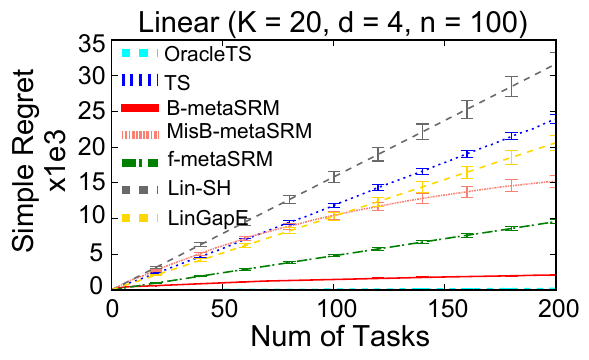}
    \end{minipage}
    \hfill
    \begin{minipage}{.45\textwidth}
        \includegraphics[width=\textwidth]{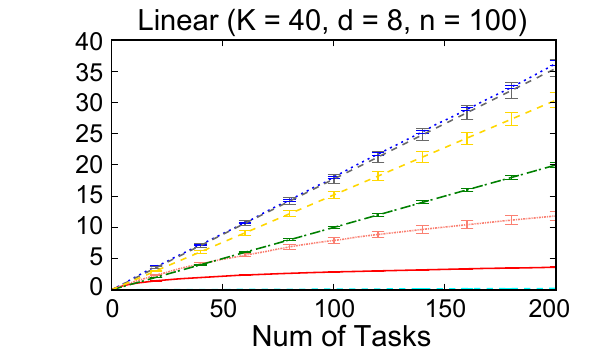}
    \end{minipage}
    \hfill
    \begin{minipage}{.45\textwidth}
        \includegraphics[width=\textwidth]{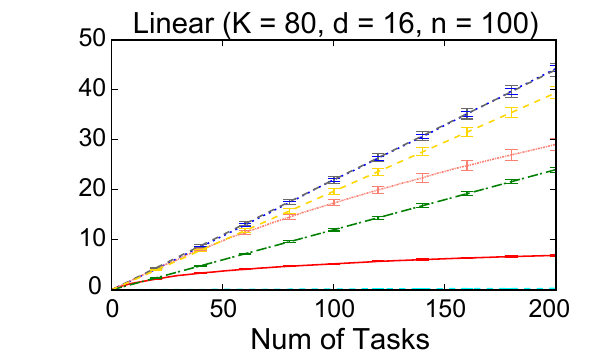}
    \end{minipage}
    \caption{Learning curves for linear Gaussian bandit experiments. The error bars are standard deviations from 100 runs.}
    \label{fig:Lin}
\end{figure*}

\section{Experiments}
\label{sec:experiment}

In this section, we empirically compare our algorithms by their \emph{average meta simple regret} over $100$ simulation runs. In each run, the prior is sampled i.i.d.\ from a fixed meta-prior. Then the algorithms run on tasks sampled i.i.d.\ from the prior. Therefore, the average simple regret is a finite-sample approximation of the Bayesian meta simple regret. Alternatively, we evaluate the algorithms based on their frequentist regret in \cref{app:FurtherExperiments}. We also experiment with a real-world dataset in \cref{app:mnist}.

We evaluate three variants of our algorithms with TS as \alg; (1) \fMetaSRM (\cref{alg:f-MBAI}) as a frequentist Meta TS. 
We tune $m_0$ and report the point-wise best performance for each task. 
(2) \bMetaSRM (\cref{alg:B-MBAI}) as a Bayesian Meta-learning algorithm. (3) \MisbMetaSRM which is the same as \bMetaSRM except that the meta-prior mean is perturbed by uniform noise from $[-50, 50]$. This is to show how a major meta-prior misspecification affects our Bayesian algorithm. The actual meta-prior is $\Ngaus(0,\Sigma_q)$. 

We do experiments with Gaussian rewards, and thus the following are our baseline for both MAB and linear bandit experiments. The first baseline is \OTS, which is TS with the correct prior $\Ngaus(\theta_*,\Sigma_0)$. Because of that, it performs the best in hindsight. The second baseline is agnostic \TS, which ignores the structure of the problem. We implement it with a prior $\Ngaus(\mathbf{0}_K, \Sigma_q+\Sigma_0)$, since $\mu_s$ can be viewed as a sample from this prior when the task structure is ignored. Note that $\Sigma_q$ is the meta-prior covariance in \cref{sec:LinGaus}.

The next set of baselines are state-of-the-art BAI algorithms. As mentioned in \cref{sec:related-work}, the goal of BAI is not SRM but it is closely related. A BAI algorithm is expected to have small simple regret for a single task. Therefore, if our algorithms outperform them, the gain must be due to meta-learning. We include sequential halving (\SH) and its linear variant (\LinSH), which are special cases of \GSE \citep{azizi2021fixedbudget}, as the state-of-the-art fixed-budget BAI algorithms. We also include \LinGapE \citep{xu2017fully-LinGapE} as it shows superior SRM performance compared to \LinSH. All experiments have $m = 200$ tasks with $n = 100$ rounds in each. \cref{app:FurtherExperiments} describes the experimental setup in more detail and also includes additional results.

\subsection{Gaussian MAB}\label{sec:GausExp}

We start our experiments with a Gaussian bandit. Specifically, we assume that $\A=[K]$ are $K$ arms with a Gaussian reward distribution $\nu_s(a;\mu_s)=\Ngaus(\mu_s(a),10^2)$, so $\sigma=10$. The mean reward is sampled as $\mu_s\sim P_{\theta_*}=\Ngaus(\theta_*,0.1^2I_K)$, so $\Sigma_0=0.1^2I_K$. The prior parameter is sampled from meta-prior as $\theta_* \sim Q = \Ngaus(\mathbf{0}_K, I_K)$, i.e., $\Sigma_q=I_K$.

\cref{fig:Gaus} shows the results for various values of $K$. We clearly observe that the meta-learning algorithms adapt to the task prior and outperform \TS. Both \fMetaSRM and \bMetaSRM perform similarly close to \OTS, which confirms the negligible cost of learning the prior as expected in our bounds. We also note that \fMetaSRM outperforms \MisbMetaSRM, which highlights the reliance of the Bayesian algorithm on a good meta-prior. \SH matches the performance of the meta-learning algorithms when $K = 4$. However, as the task becomes harder ($K > 4$), it underperforms our algorithms significantly. For smaller $K$, the tasks share less information and thus meta-learning does not improve the learning as much.

\subsection{Linear Gaussian Bandits}\label{sec:LinExp}

Now take a linear bandit (\cref{sec:LinGaus}) in $d$ dimensions with $K=5d$ arms, the arms are sampled from a unit sphere uniformly. The reward of arm $a$ is distributed as $\Ngaus(a\T\mu_s,10^2)$, so $\sigma=10$, and $\mu_s$ is sampled from $P_{*}=\Ngaus(\theta_*, 0.1^2 I_d)$, so $\Sigma_0=0.1^2 I_d$. The prior parameter, $\theta_*$, is sampled from meta-prior $Q = \Ngaus(\mathbf{0}_d, I_d)$, so $\Sigma_q=I_d$.

\cref{fig:Lin} shows experiments for various values of $d$. As expected, larger $d$ increase the regret of all the algorithms. Compared to \cref{sec:GausExp}, the problem of learning the prior is more difficult, and the gap of \bMetaSRM and \OTS increases. \fMetaSRM also outperforms \TS, but it has a much higher regret than \bMetaSRM. While \MisbMetaSRM under-performs \fMetaSRM in the MAB tasks, it performs closer to \bMetaSRM in this experiment. The BAI algorithms, \LinSH and \LinGapE, under-perform our meta-learning algorithms and are closer to \TS than in \cref{fig:Gaus}. The value of knowledge transfer in the linear setting is higher since the linear model parameter is shared by many arms.

Our linear bandit experiment confirms the applicability of our algorithms to structured problems, which shows potential for solving real-world problems. Specifically, the success of \MisbMetaSRM confirms the robustness of \bMetaSRM to misspecification.

\section{Conclusions and Future Work}

We develop a meta-learning framework for SRM, where the agent improves by interacting repeatedly with similar tasks. We propose two algorithms: a Bayesian algorithm that maintains a distribution over task parameters and the frequentist one that estimates the task parameters using frequentist methods. The Bayesian algorithm has superior regret guarantees while the frequentist one can be applied to a larger family of problems.

This work lays foundations for Bayesian SRM and readily extends to reinforcement learning (RL). For instance, we can extend our framework to task structures, such as parallel or arbitrarily ordered \citep{wan2021metadatabased,hong2022hierarchical}. Our Bayesian algorithm easily extends to tabular and factored MDPs RL \citep{lu2019information}. Also, the frequentist algorithm applies to POMDPs \citep{simchowitz2021bayesian}.

\bibliography{refs.bib,References.bib}

\begin{thebibliography}{43}
\providecommand{\natexlab}[1]{#1}
\providecommand{\url}[1]{\texttt{#1}}
\expandafter\ifx\csname urlstyle\endcsname\relax
  \providecommand{\doi}[1]{doi: #1}\else
  \providecommand{\doi}{doi: \begingroup \urlstyle{rm}\Url}\fi

\bibitem[Abeille and Lazaric(2017)]{abeille17linear}
Marc Abeille and Alessandro Lazaric.
\newblock Linear {Thompson} sampling revisited.
\newblock In \emph{Proceedings of the 20th International Conference on
  Artificial Intelligence and Statistics}, 2017.

\bibitem[Agrawal and Goyal(2012)]{agrawal2012analysis}
Shipra Agrawal and Navin Goyal.
\newblock Analysis of thompson sampling for the multi-armed bandit problem.
\newblock In \emph{Conference on learning theory}, pages 39--1. JMLR Workshop
  and Conference Proceedings, 2012.

\bibitem[Agrawal and Goyal(2013)]{agrawal2013further}
Shipra Agrawal and Navin Goyal.
\newblock Further optimal regret bounds for thompson sampling.
\newblock In \emph{Artificial intelligence and statistics}, pages 99--107.
  PMLR, 2013.

\bibitem[Alieva et~al.(2021)Alieva, Cutkosky, and Das]{pmlr-v139-alieva21a}
Ayya Alieva, Ashok Cutkosky, and Abhimanyu Das.
\newblock Robust pure exploration in linear bandits with limited budget.
\newblock In Marina Meila and Tong Zhang, editors, \emph{Proceedings of the
  38th International Conference on Machine Learning}, volume 139 of
  \emph{Proceedings of Machine Learning Research}, pages 187--195. PMLR, 18--24
  Jul 2021.
\newblock URL \url{https://proceedings.mlr.press/v139/alieva21a.html}.

\bibitem[Audibert and Bubeck(2010)]{audibert-2010-BAI}
Jean-Yves Audibert and S{\'e}bastien Bubeck.
\newblock {Best Arm Identification in Multi-Armed Bandits}.
\newblock In \emph{{COLT - 23th Conference on Learning Theory - 2010}}, page 13
  p., Haifa, Israel, June 2010.
\newblock URL \url{https://hal-enpc.archives-ouvertes.fr/hal-00654404}.

\bibitem[Azar et~al.(2013)Azar, Lazaric, and Brunskill]{azar13sequential}
Mohammad~Gheshlaghi Azar, Alessandro Lazaric, and Emma Brunskill.
\newblock Sequential transfer in multi-armed bandit with finite set of models.
\newblock In \emph{Advances in Neural Information Processing Systems 26}, pages
  2220--2228, 2013.

\bibitem[Azizi et~al.(2022)Azizi, Kveton, and
  Ghavamzadeh]{azizi2021fixedbudget}
MohammadJavad Azizi, Branislav Kveton, and Mohammad Ghavamzadeh.
\newblock Fixed-budget best-arm identification in structured bandits.
\newblock In \emph{Proceedings of the Thirty-First International Joint
  Conference on Artificial Intelligence, IJCAI-22}, pages 2798--2804, 2022.

\bibitem[Bastani et~al.(2019)Bastani, Simchi-Levi, and Zhu]{bastani19meta}
Hamsa Bastani, David Simchi-Levi, and Ruihao Zhu.
\newblock Meta dynamic pricing: Transfer learning across experiments.
\newblock \emph{Management Science}, 2019.
\newblock \doi{10.1287/mnsc.2021.4071}.
\newblock URL \url{https://doi.org/10.1287/mnsc.2021.4071}.

\bibitem[Basu et~al.(2021)Basu, Kveton, Zaheer, and Szepesv{\'a}ri]{basu2021no}
Soumya Basu, Branislav Kveton, Manzil Zaheer, and Csaba Szepesv{\'a}ri.
\newblock No regrets for learning the prior in bandits.
\newblock \emph{Advances in Neural Information Processing Systems}, 34, 2021.

\bibitem[Baxter(1998)]{baxter98theoretical}
Jonathan Baxter.
\newblock Theoretical models of learning to learn.
\newblock In \emph{Learning to Learn}, pages 71--94. Springer, 1998.

\bibitem[Baxter(2000)]{baxter00model}
Jonathan Baxter.
\newblock A model of inductive bias learning.
\newblock \emph{Journal of Artificial Intelligence Research}, 12:\penalty0
  149--198, 2000.

\bibitem[Boutilier et~al.(2020)Boutilier, Hsu, Kveton, Mladenov, Szepesvari,
  and Zaheer]{boutilier20differentiable}
Craig Boutilier, Chih-Wei Hsu, Branislav Kveton, Martin Mladenov, Csaba
  Szepesvari, and Manzil Zaheer.
\newblock Differentiable meta-learning of bandit policies.
\newblock In \emph{Advances in Neural Information Processing Systems 33}, 2020.

\bibitem[Cella et~al.(2020)Cella, Lazaric, and Pontil]{cella20metalearning}
Leonardo Cella, Alessandro Lazaric, and Massimiliano Pontil.
\newblock Meta-learning with stochastic linear bandits.
\newblock In \emph{Proceedings of the 37th International Conference on Machine
  Learning}, 2020.

\bibitem[Deshmukh et~al.(2017)Deshmukh, Dogan, and Scott]{deshmukh17multitask}
Aniket~Anand Deshmukh, Urun Dogan, and Clayton Scott.
\newblock Multi-task learning for contextual bandits.
\newblock In \emph{Advances in Neural Information Processing Systems 30}, pages
  4848--4856, 2017.

\bibitem[Duan et~al.(2016)Duan, Schulman, Chen, Bartlett, Sutskever, and
  Abbeel]{duan16rl2}
Yan Duan, John Schulman, Xi~Chen, Peter~L Bartlett, Ilya Sutskever, and Pieter
  Abbeel.
\newblock {RL}$^2$: Fast reinforcement learning via slow reinforcement
  learning.
\newblock \emph{arXiv preprint arXiv:1611.02779}, 2016.

\bibitem[Finn et~al.(2018)Finn, Xu, and Levine]{finn18probabilistic}
Chelsea Finn, Kelvin Xu, and Sergey Levine.
\newblock Probabilistic model-agnostic meta-learning.
\newblock In \emph{Advances in Neural Information Processing Systems 31}, pages
  9537--9548, 2018.

\bibitem[Gabillon et~al.(2012)Gabillon, Ghavamzadeh, and
  Lazaric]{Gabillon-2012}
Victor Gabillon, Mohammad Ghavamzadeh, and Alessandro Lazaric.
\newblock Best arm identification: A unified approach to fixed budget and fixed
  confidence.
\newblock In F.~Pereira, C.~J.~C. Burges, L.~Bottou, and K.~Q. Weinberger,
  editors, \emph{Advances in Neural Information Processing Systems}, volume~25,
  pages 3212--3220. Curran Associates, Inc., 2012.
\newblock URL
  \url{https://proceedings.neurips.cc/paper/2012/file/8b0d268963dd0cfb808aac48a549829f-Paper.pdf}.

\bibitem[Gentile et~al.(2014)Gentile, Li, and Zappella]{gentile14online}
Claudio Gentile, Shuai Li, and Giovanni Zappella.
\newblock Online clustering of bandits.
\newblock In \emph{Proceedings of the 31st International Conference on Machine
  Learning}, pages 757--765, 2014.

\bibitem[Hoffman et~al.(2014)Hoffman, Shahriari, and
  Freitas]{pmlr-v33-hoffman14}
Matthew Hoffman, Bobak Shahriari, and Nando Freitas.
\newblock On correlation and budget constraints in model-based bandit
  optimization with application to automatic machine learning.
\newblock In \emph{Artificial Intelligence and Statistics}, pages 365--374.
  PMLR, 2014.

\bibitem[Hong et~al.(2022)Hong, Kveton, Zaheer, and
  Ghavamzadeh]{hong2022hierarchical}
Joey Hong, Branislav Kveton, Manzil Zaheer, and Mohammad Ghavamzadeh.
\newblock Hierarchical {Bayesian} bandits.
\newblock In \emph{International Conference on Artificial Intelligence and
  Statistics}, pages 7724--7741. PMLR, 2022.

\bibitem[Hsu et~al.(2019)Hsu, Kveton, Meshi, Mladenov, and
  Szepesvari]{hsu19empirical}
Chih-Wei Hsu, Branislav Kveton, Ofer Meshi, Martin Mladenov, and Csaba
  Szepesvari.
\newblock Empirical bayes regret minimization.
\newblock \emph{arXiv preprint arXiv:1904.02664}, 2019.

\bibitem[Kaufmann et~al.(2016)Kaufmann, Capp{\'{e}}, and
  Garivier]{kaufmann16complexity}
Emilie Kaufmann, Olivier Capp{\'{e}}, and Aur{\'{e}}lien Garivier.
\newblock On the complexity of best-arm identification in multi-armed bandit
  models.
\newblock \emph{JMLR}, 17:\penalty0 1:1--1:42, 2016.

\bibitem[Komiyama et~al.(2021)Komiyama, Ariu, Kato, and Qin]{komiyama21optimal}
Junpei Komiyama, Kaito Ariu, Masahiro Kato, and Chao Qin.
\newblock Optimal simple regret in bayesian best arm identification.
\newblock \emph{arXiv preprint arXiv:2111.09885}, 2021.

\bibitem[Kuleshov and Precup(2014)]{kuleshov14algorithms}
Volodymyr Kuleshov and Doina Precup.
\newblock Algorithms for multi-armed bandit problems.
\newblock \emph{CoRR}, abs/1402.6028, 2014.
\newblock URL \url{http://arxiv.org/abs/1402.6028}.

\bibitem[Kveton et~al.(2020)Kveton, Mladenov, Hsu, Zaheer, Szepesvari, and
  Boutilier]{kveton20differentiable}
Branislav Kveton, Martin Mladenov, Chih-Wei Hsu, Manzil Zaheer, Csaba
  Szepesvari, and Craig Boutilier.
\newblock Differentiable meta-learning in contextual bandits.
\newblock \emph{arXiv e-prints}, pages arXiv--2006, 2020.

\bibitem[Kveton et~al.(2021)Kveton, wei Hsu, Boutilier, Szepesvari, Zaheer,
  Mladenov, and Konobeev]{kveton2021metathompson}
Branislav Kveton, Chih wei Hsu, Craig Boutilier, Csaba Szepesvari, Manzil
  Zaheer, Martin Mladenov, and Michael Konobeev.
\newblock Meta-thompson sampling.
\newblock In \emph{Proceedings of the 38th International Conference on Machine
  Learning (ICML 2021)}, pages 5884--5893, 2021.

\bibitem[Lattimore and Szepesv{\'a}ri(2020)]{lattimore-Bandit}
Tor Lattimore and Csaba Szepesv{\'a}ri.
\newblock \emph{Bandit algorithms}.
\newblock Cambridge University Press, 2020.

\bibitem[Liu et~al.(2022)Liu, Devraj, Van~Roy, and Xu]{liu2022gaussian}
Yueyang Liu, Adithya~M Devraj, Benjamin Van~Roy, and Kuang Xu.
\newblock Gaussian imagination in bandit learning.
\newblock \emph{arXiv preprint arXiv:2201.01902}, 2022.

\bibitem[Lu and Van~Roy(2019)]{lu2019information}
Xiuyuan Lu and Benjamin Van~Roy.
\newblock Information-theoretic confidence bounds for reinforcement learning.
\newblock In \emph{Advances in Neural Information Processing Systems},
  volume~32, 2019.

\bibitem[Maes et~al.(2012)Maes, Wehenkel, and Ernst]{maes12metalearning}
Francis Maes, Louis Wehenkel, and Damien Ernst.
\newblock Meta-learning of exploration/exploitation strategies: The multi-armed
  bandit case.
\newblock In \emph{Proceedings of the 4th International Conference on Agents
  and Artificial Intelligence}, pages 100--115, 2012.

\bibitem[Mason et~al.(2020)Mason, Jain, Tripathy, and
  Nowak]{NEURIPS2020_edf0320a}
Blake Mason, Lalit Jain, Ardhendu Tripathy, and Robert Nowak.
\newblock Finding all \textbackslash epsilon-good arms in stochastic bandits.
\newblock In H.~Larochelle, M.~Ranzato, R.~Hadsell, M.~F. Balcan, and H.~Lin,
  editors, \emph{Advances in Neural Information Processing Systems}, volume~33,
  pages 20707--20718. Curran Associates, Inc., 2020.
\newblock URL
  \url{https://proceedings.neurips.cc/paper/2020/file/edf0320adc8658b25ca26be5351b6c4a-Paper.pdf}.

\bibitem[Min et~al.(2020)Min, Moallemi, and Russo]{min20policy}
Seungki Min, Ciamac~C Moallemi, and Daniel~J Russo.
\newblock Policy gradient optimization of thompson sampling policies.
\newblock \emph{arXiv preprint arXiv:2006.16507}, 2020.

\bibitem[R{\'e}da et~al.(2021)R{\'e}da, Kaufmann, and
  Delahaye-Duriez]{reda2021top}
Cl{\'e}mence R{\'e}da, Emilie Kaufmann, and Andr{\'e}e Delahaye-Duriez.
\newblock Top-m identification for linear bandits.
\newblock In \emph{International Conference on Artificial Intelligence and
  Statistics}, pages 1108--1116. PMLR, 2021.

\bibitem[Russo(2020)]{russo2020simple}
Daniel Russo.
\newblock Simple bayesian algorithms for best-arm identification.
\newblock \emph{Operations Research}, 68\penalty0 (6):\penalty0 1625--1647,
  2020.

\bibitem[Simchowitz et~al.(2021)Simchowitz, Tosh, Krishnamurthy, Hsu, Lykouris,
  Dudik, and Schapire]{simchowitz2021bayesian}
Max Simchowitz, Christopher Tosh, Akshay Krishnamurthy, Daniel~J Hsu, Thodoris
  Lykouris, Miro Dudik, and Robert~E Schapire.
\newblock Bayesian decision-making under misspecified priors with applications
  to meta-learning.
\newblock \emph{Advances in Neural Information Processing Systems}, 34, 2021.

\bibitem[Thrun(1996)]{thrun96explanationbased}
Sebastian Thrun.
\newblock \emph{Explanation-Based Neural Network Learning - A Lifelong Learning
  Approach}.
\newblock PhD thesis, University of Bonn, 1996.

\bibitem[Thrun(1998)]{thrun98lifelong}
Sebastian Thrun.
\newblock Lifelong learning algorithms.
\newblock In \emph{Learning to Learn}, pages 181--209. Springer, 1998.

\bibitem[Tripathi et~al.(1994)Tripathi, Gupta, and
  Gurland]{tripathi1994estimation}
Ram~C Tripathi, Ramesh~C Gupta, and John Gurland.
\newblock Estimation of parameters in the beta binomial model.
\newblock \emph{Annals of the Institute of Statistical Mathematics},
  46\penalty0 (2):\penalty0 317--331, 1994.

\bibitem[Vermorel and Mohri(2005)]{vermorel05multiarmed}
Joannes Vermorel and Mehryar Mohri.
\newblock Multi-armed bandit algorithms and empirical evaluation.
\newblock In \emph{Proceedings of the 16th European Conference on Machine
  Learning}, pages 437--448, 2005.

\bibitem[Vershynin(2018)]{Vershynin-HDP-2019}
Roman Vershynin.
\newblock \emph{High-dimensional probability: An introduction with applications
  in data science}, volume~47.
\newblock Cambridge university press, 2018.

\bibitem[Wan et~al.(2021)Wan, Ge, and Song]{wan2021metadatabased}
Runzhe Wan, Lin Ge, and Rui Song.
\newblock Metadata-based multi-task bandits with {Bayesian} hierarchical
  models.
\newblock In A.~Beygelzimer, Y.~Dauphin, P.~Liang, and J.~Wortman Vaughan,
  editors, \emph{Advances in Neural Information Processing Systems}, 2021.
\newblock URL \url{https://openreview.net/forum?id=nW4xl2CjcVg}.

\bibitem[Xu et~al.(2018)Xu, Honda, and Sugiyama]{xu2017fully-LinGapE}
Liyuan Xu, Junya Honda, and Masashi Sugiyama.
\newblock Fully adaptive algorithm for pure exploration in linear bandits,
  2018.

\bibitem[Yang and Toni(2020)]{yang20differentiable}
Kaige Yang and Laura Toni.
\newblock Differentiable linear bandit algorithm.
\newblock \emph{arXiv preprint arXiv:2006.03000}, 2020.

\end{thebibliography}

\newpage
\appendix

\section{Further Setting Details}
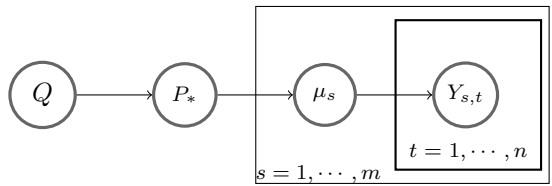
\begin{figure}
\centering
\begin{tikzpicture}[
roundnode/.style={circle, draw=black!60, fill=white!5, very thick, minimum size=5mm, inner sep=2pt, text width=4ex,text centered}
]
\node[roundnode](q){$Q$};
\node[roundnode](ps)[right=of q]{\footnotesize{$P_*$}};
\node[roundnode](mus)[right=of ps]{\footnotesize{$\mu_s$}};
\node[roundnode](ys)[right=of mus]{\footnotesize{$Y_{s,t}$}};

\draw[->] (q.east) -- (ps.west);
\draw[->] (ps.east) -- (mus.west);
\draw[->] (mus.east) -- (ys.west);

\node[block, right=-5mm of ys.west, text height=1.7cm](a){\footnotesize{$t=1,\cdots,n$}};

\node[draw=black, shape=rectangle, inner sep=0pt, fill=none, text=black, left=-1mm of a.east, minimum width=3.5cm, minimum height = 1.4cm, text height=2.3cm, text width=3.9cm] {\footnotesize{$s=1,\cdots,m$}};
\end{tikzpicture}
\caption{Generative model of the meta learning SRM setting studied in the paper. Note that there is no meta prior $Q$ in the frequentist setting.}
\label{fig:set}
\end{figure}

\cref{fig:set} illustrate the generative model of the meta learning SRM setting studied in this paper. Note that there is no meta prior $Q$ in the frequentist setting.

\section{Cumulative to Simple Regret}
\label{app:Cum2Sim}

In this section, we propose a general framework for cumulative regret to simple regret reduction that establishes many new algorithms and leads to efficient SRM methods. We use this simple but fundamentally important tool in our proofs. In the frequentist analysis, this is used to bound the regret of the base algorithm. We also use this in the full regret reduction in the Bayesian setting.

Fix a task $s$ and consider an algorithm that pulls a sequence of arms, $(A_{s,t})_{t\in[n]}\;$.  
Now let its per-task cumulative regret with prior $P$ be 
\begin{align*}
    R_s(n, P):=\E_{\mu_s\sim P}\E\left[n\mu(A^*) - \sum_{t=1}^n\mu(A_t)\right]\;.
\end{align*}
where the inner expectation is taken over the algorithmic and rewards randomness.
Now suppose at the end of the task, we choose arm $a$ with probability $\rho_s(a)=\frac{N_{a,s}}{n}$ and declare it to be the best arm, $\Ahat_s$. Then the following result bounds the per-task simple regret of this general procedure based on its per-task cumulative regret.

\cumTosim*
\begin{proof}[Proof of \cref{prop:cum2sim}]
    Fix a task $s$. 
    We can rewrite its per-task simple regret as
    \begin{align*}
        \SR_s(n, P)&={\E}_{\mu_s\sim P}\E\left[\mu_s(A^*)-\sum_{a\in\A}\frac{N_{a,s}}{n}\mu_s(a)\right]
        \\&={\E}_{\mu_s\sim P}\E\left[\sum_{a\in\A}\frac{\mu_s(A^*)}{n} -\frac{N_{a,s}}{n}\mu_s(a)\right]
        \\&=\frac{R_s(n, P)}{n}.
    \end{align*} 
    where the first equality holds by the nature of the procedure, and the last one used the linearity of expectation twice.
\end{proof}
It is also straightforward to see that \cref{prop:cum2sim} works for either frequentist meta simple regret or Bayesian meta simple regret. This is because the former is the summation of $\SR_s$ over tasks, and the latter is achieved by taking an expectation of the former over $P_*$.

\section{Bayesian Analysis}\label{app:BayesReg-pfs}

We defined $\tau_s=(A_{s,1}, Y_{s,1},\cdots, A_{s,n},Y_{s,n})$ to be the trajectory of task $s$, $\tau_{1:s}=\oplus_{\ell=1}^s \tau_{\ell}$ be the trajectory of tasks $1$ to $s$, and $\tau_{1:s,t}$ be the trajectory from the beginning of the first task up to round $t-1$ of task $s$. Let $\E_{s,t}[\cdot]=\E[\cdot|\tau_{1:s,t}]$. We define $\Gamma_{s,t}$ and $\beta_{s,t}$ to be the potentially trajectory-dependent non-negative random variables, such that the following inequality holds:
\begin{align}
    \E_{s,t}[\mu(A^*_s)&-\mu(A_{s,t})] 
    \leq \Gamma_{s,t}\sqrt{\Info_{s,t}(\mu_s;A_{s,t},Y_{s,t})}+\beta_{s,t},\label{eq:info}
\end{align}
where $\Info_{s,t}(\mu_s;A_{s,t},Y_{s,t})$ is the mutual information of the mean reward of task $s$ and the pair of arm taken $A_{s,t}$ and reward observed $Y_{s,t}$ in round $t$ of task $s$, conditioned on the trajectory $\tau_{1:s,t}$. These random variables are well-defined as introduced in \citet{lu2019information}.  

For \ucb, we use the upper bound
\begin{align}
    U_{s,t}(a):=\E_{s,t}[Y_{s,t}(a)]+\Gamma_{s,t}\sqrt{\Info_{s,t}(\mu_{s};A_{s,t},Y_{s,t}(a))}\;,\label{eq:Ucb}
\end{align}
The quantity $\E_{s,t}[Y_{s,t}(a)]$ is calculated based on the posterior of $\mu_s$ at round $t$. 

We remind some notation in their general form. If $\frac{\da P}{\da Q}$ is the Radon-Nikodym derivative of $P$ with respect to $Q$, we know it is finite when $P$ is absolutely continuous with respect to $Q$. Let $D(P\parr Q)=\int\log(\frac{\da P}{\da Q})\da P$ be the relative entropy of $P$ with respect to $Q$. Also, Let $\Info(X;Y)=D(\P(X,Y)\parr \P(X)\P(Y))$ be the mutual information between $X$ and $Y$ and $\Info_{s,t}(X;Y):=\Info(X;Y|\tau_{1:s,t})$ be the same mutual information given trajectory $\tau_{1:s,t}$. We also define the \emph{conditional mutual information} between $X$ and $Y$ conditioned on $Z$. We define this quantity as $\Info(X; Y | Z) = \E[\hat{\Info}(X; Y | Z)]$, where $\hat{\Info}(X; Y | Z) = D(P(X, Y | Z)\parr P(X | Z)P(Y | Z))$ is the \emph{random conditional mutual information} between $X$ and $Y$ given $Z$. Note that $\hat{\Info}(X; Y | Z)$ is a
function of $Z$. By the \emph{chain rule} for the random conditional mutual information and taking the expectation over $Y_2 | Z$ we get $\hat{\Info}(X; Y_1, Y_2 | Z) = \E[\hat{\Info}(X; Y_1 | Y_2, Z) | Z] + \hat{\Info}(X; Y_2 | Z)$. Without $Z$, the usual chain rule is $\Info(X; Y_1, Y_2) = \Info(X; Y_1 | Y_2) + \Info(X; Y_2)$.

\infoBound*

\begin{proof}
Fixing a prior $P_*$ and summing over $s\in[m]$, as the reduction from cumulative to simple regret in \cref{prop:cum2sim} holds for any prior, $\SR_s(m,n,P_*)=\frac{1}{n} \sum_{s=1}^m R_s(n,P_*)$. 
Therefore, by taking expectation over $P_*\sim Q$, we know $\BSR(m,n)= \frac{1}{n}\E_{P_*\sim Q}[\sum_{s=1}^m R_s(n,P_*)]$. Now notice that $\E_{P_*\sim Q}[\sum_{s=1}^m R_s(n,P_*)]$ is bounded by Lemma 2 of \citet{basu2021no} as follows
\begin{align*}
    \SR(m,n, P_*)\leq \Gamma\sqrt{mn\Info(\theta_*;\tau_{1:m})}+\sum_{s=1}^m\Gamma_s\sqrt{n\Info(\mu_s;\tau_s|\theta_*,\tau_{1:s-1})}+\sum_{s=1}^m\sum_{t=1}^n \E\beta_{s,t}\;.
\end{align*}

Now, we only need to divide the right-hand side by $n$.

\end{proof}

\subsection{Proof of Bayesian Linear Bandit}\label{app:LinBan-Bay-pfs}

\corLinBanBayes*

\begin{proof}[Proof of \cref{cor:linBan-Bayes}]
    This is only applying \cref{prop:cum2sim} to Theorem 5 of \citet{basu2021no}. Note that we can directly get this result from the generic Bayesian meta simple regret bound \cref{lem:infoBd} by setting $\Gamma_{s,t}$ and $\beta_{s,t}$ properly based on the properties of linear Gaussian bandits environment from \citet{lu2019information}.   
\end{proof}

\subsection{Information Theoretic Technical Tools}
The conditional entropy terms are defined as follows:
\begin{align*}
  h_{s,t}(\mu_s )
  & = \E_{s,t}\left[-\log\left(\mathbb{P}_{s,t}(\mu_s )\right)\right]\,, \\
  \quad h_{s,t}(\mu_{*})
  & = \E_{s,t}\left[-\log\left(\mathbb{P}_{s,t}(\mu_{*})\right)\right]\,, \\
  h_{s,t}(\mu_s \mid\mu_{*})
  & = \E_{s,t}\left[-\log\left(\mathbb{P}_{s,t}(\mu_s \mid\mu_{*})\right)\right]\,.
\end{align*}

Therefore, all the different mutual information terms $I_{s,t}(\cdot; A_{s,t}, Y_{s,t})$, and the entropy terms $h_{s,t}(\cdot)$ are random variables that depends on the history $\tau_{1:s,t}$. 

We next state some entropy and mutual information relationships which we use later.
\begin{proposition}\label{prop:chain}
For all $s$, $t$, and any history $H_{1:s,t}$, the following hold
\begin{align*}
  I_{s,t}(\mu_s, \mu_{*}; A_{s,t}, Y_{s,t})
  & = I_{s,t}(\mu_{*}; A_{s,t}, Y_{s,t}) + I_{s,t}(\mu_s; A_{s,t}, Y_{s,t}\mid\mu_{*})\,, \\
  I_{s,t}(\mu_s; A_{s,t}, Y_{s,t})
  & = h_{s,t}(\mu_s) - h_{s,t+1}(\mu_s)\,.
\end{align*}
\end{proposition}

\subsection{Bayesian UCB}\label{sec:ucb-proofs}

Let's consider a UCB with $U_{s,t}(a):=\E_{s,t}[Y_{s,t}(a)]+\Gamma_{s,t}\sqrt{\Info_{s,t}(\mu_{s};A_{s,t},Y_{s,t}(a))}$, where $\E_{s,t}[\cdot]=\E[\cdot|\tau_{1:s,t}]$. We call this \ucb. The $\E_{s,t}[Y_{s,t}(a)]$ is calculated based on the posterior of $\mu_s$ at round $t$. In the linear bandits setting, $\E_{s,t}[Y_{s,t}(a)]=a\T\hat{\mu}_{s,t}$ where $\hat{\mu}_{s,t}\sim\Ngaus(\hat{\theta}_{s,t},\hat{\Sigma}_{s,t})$ is a sample from the posterior of $\mu$, for 
\begin{align*}
  \hat{\theta}_{s,t}
  = \hat{\Sigma}_{s,t}\left((\Sigma_{0}+\hat{\Sigma}_{s})^{-1}\hat{\mu}_{s} + \sum_{\ell=1}^{t-1} A_{s,\ell} Y_{s,\ell}\right)\,, \quad
  \hat{\Sigma}_{s,t}^{-1}
  = (\Sigma_{0}+ \hat{\Sigma}_{s})^{-1} + \sum_{\ell=1}^{t-1}\tfrac{A_{s,\ell}A^T_{s,\ell}}{\sigma^2}\,,
\end{align*}

The following holds for \ucb algorithm, which is the analogous of Lemma 3 of \citet{basu2021no} for TS. 

\begin{restatable}[]{lem}{concentration}
\label{lemm:concentration} 
For all tasks $s\in [m]$, rounds $t\in [n]$, and any $\delta\in (0,1]$, for \cref{alg:B-MBAI} with \ucb, \cref{eq:info} holds almost surely for
\begin{align*}
  \Gamma_{s,t}
  = 4\sqrt{\frac{\sigma^2_{\max}(\hat{\Sigma}_{s,t})}{\log(1+\sigma^2_{\max}(\hat{\Sigma}_{s,t})/\sigma^2)}\log\frac{4|\ca{A}|}{\delta}}\,, \quad
  \beta_{s,t}
  = \frac{\delta}{2}\max_{a\in\A}\|a\|_2\E_{s,t}[\|\mu_s\|_2]\,.
\end{align*} 
Moreover, for each task $s$, the following history-independent bound holds almost surely,
\begin{align}
  \sigma^2_{\max}(\hat{\Sigma}_{s,t})
  \leq \lambda_1(\Sigma_{0})\left(1+\tfrac{\lambda_{1}(\Sigma_q)(1+  \tfrac{ \sigma^2 }{\eta\lambda_{1}(\Sigma_0)})}{\lambda_{1}(\Sigma_0) +  \sigma^2/ \eta + \sqrt{s}\lambda_{1}(\Sigma_q)}\right)\,.
  \label{eq:smsqbound}
\end{align}
\end{restatable}
\begin{proof}[Proof of \cref{lemm:concentration}.]
    
    Let's define 
    \[\ca{M}_{s,t}:=\bigg\{\mu:\abs{a\T\mu-\E_{s,t}[a\T\mu_{s}]}\leq \frac{\Gamma_{s,t}}{2}\sqrt{\Info_{s,t}(\mu_{s};A_{s,t},Y_{s,t}(a))}\bigg\}\]
    
    We can characterize the trajectory dependent conditional mutual entropy of $\mu_s $ given the history $\tau_{1:s,t}$ as 
    \begin{align*}
        I_{s,t}(\mu_s;A_{s,t},Y_{s,t})
        &= h_{s,t}(\mu_s) - h_{s,t+1}(\mu_s )\\
        &= \tfrac{1}{2}\log(\det(2\pi e (\hat{\Sigma}_{s,t-1}))) - \tfrac{1}{2}\log(\det(2\pi e \hat{\Sigma}_{s,t}))\\
        &= \tfrac{1}{2}\log(\det(\hat{\Sigma}_{s,t-1}\hat{\Sigma}^{-1}_{s,t}))\\
        &= \tfrac{1}{2}\log\left(\det\left(I+\hat{\Sigma}_{s,t-1}\tfrac{A_{s,t}A_{s,t}^T}{\sigma^2}\right)\right)\\
        &= \tfrac{1}{2}\log\left(\det\left(1+\tfrac{A_{s,t}^T\hat{\Sigma}_{s,t-1}A_{s,t}}{\sigma^2}\right)\right)
    \end{align*} 
    Where the last equality uses the matrix determinant lemma.\footnote{For an invertible square matrix $A$,  and vectors $u$ and $v$, by matrix determinant lemma we know
    $\det \left(A + uv^T\right)=\left(1+ v^T A^{-1} u \right)\,\det \left(A\right).$ We use $A = I$, $u = \hat{\Sigma}_{s,t-1}A_{s,t}$, and $v = A_{s,t}/\sigma^2$.}
    Recall that $\sigma^2_{\max}(\hat{\Sigma}_{s,t})  =\max_{a\in \ca{A}} a^T \hat{\Sigma}_{s,t} a$ for all $s\leq m$ and $t\leq n$. For $\delta\in (0,1]$, let 
    $$\Gamma_{s,t} = 4\sqrt{\frac{\sigma^2_{\max}(\hat{\Sigma}_{s,t-1})}{\log(1+\sigma^2_{\max}(\hat{\Sigma}_{s,t-1})/\sigma^2 )}\log(\tfrac{4|\ca{A}|}{\delta}}).$$ 
    Now  it follows from \citet{lu2019information} Lemma 5 that for the $\Gamma_{s,t}$ defined as above we have 
    $$P_{s,t}(\mu_s \in \ca{M}_{s,t}) \geq 1- \delta/2.$$
    
    Next we bound the gap as follows.
    \begin{align}
        \E[\Delta_{s,t}]&=\E[\Indic\{\mu_s\in\ca{M}_{s,t}\}({A^*_{s}}\T\mu_{s}-A_{s,t}\T\mu_s)]\ +\ \E[\Indic\{\mu_s\notin\ca{M}_{s,t}\}({A^*_{s}}\T\mu_{s}-A_{s,t}\T\mu_s)]\nonumber%
    \end{align}
    We know
    \begin{align}
        \E[\Indic\{\mu_s\in\ca{M}_{s,t}\}({A^*_{s}}\T\mu_{s}-A_{s,t}\T\mu_s)] &\leq \E[\Indic\{\mu_s\in\ca{M}_{s,t}\}({A^*_{s}}\T\mu_{s}-U_{s,t}(A^*_s)+U_{s,t}(A_{s,t})-A_{s,t}\T\mu_s)] \nonumber
        \\
        &\leq \E[\Indic\{\mu_s\in\ca{M}_{s,t}\}(U_{s,t}(A_{s,t})-A_{s,t}\T\mu_s)] \nonumber
        \\
        &\leq \E_{s,t}\Big[\sum_{a\in\A}\Indic\{A_{s,t}=a\}\Gamma_{s,t}\sqrt{\Info_{s,t}(\mu_{s};a,Y_{s,t}(a))}\Big] \nonumber
        \\
        &\leq  \Gamma_{s,t}\sqrt{\Info_{s,t}(\mu_{s};A_{s,t}, Y_{s,t})} \nonumber
    \end{align}
    where the last inequality used the same argument as in Lemma 3 of \citet{lu2019information} based on the conditional independence of $A_{s,t}$ and $\mu_s$ given $\tau_{1:s,t}$. We also know
    \begin{align}
        \E[\Indic\{\mu_s\notin\ca{M}_{s,t}\}({A^*_{s}}\T\mu_{s}-A_{s,t}\T\mu_s)]\leq \tfrac{\delta}{2} \E_{s,t}[\max_{a\in\A}|a\T\mu_s|]=\tfrac{\delta}{2}\max_{a\in\A}\vnormf{a}_2\E_{s,t}[\vnormf{\mu_s}_2]\nonumber
    \end{align}
    
    The second part of the proof is due to \citet{basu2021no} Lemma 3.
\end{proof}

Note that \cref{lem:infoBd} applies generically to any algorithm including \ucb, as we do not use the properties of the algorithm in its proof. 
\begin{restatable}[\textbf{Linear bandit, UCB}]{theorem}{linearinfoUCB}\label{thm:linearinfoUCB}
    The meta simple regret of \bMetaSRM with \ucb as its  with forced exploration is bounded for any $\delta \in (0,1]$ as
    \begin{align*}
           \BSR(m,n)\leq c_1\sqrt{d m /n}+ (m+c_2)\SR_{\delta}(n) + c_3 \sqrt{m/n}
    \end{align*}
    where $c_1=O(\sqrt{\log(K/\delta)\log m})$, $c_2=O(\log m)$, and $c_3$ is a constant in $m$ and $n$. Also, $\SR_{\delta}(n)$ is a special per-task simple regret which is bounded as $\SR_{\delta}(n)\leq c_4\sqrt{d /n}$, where $c_4=O\left(\sqrt{\log(K/\delta)\log n}\right)$.
\end{restatable}
\begin{proof}[Proof of \cref{thm:linearinfoUCB}]
    As shown in Theorem 5 (linear bandits) of \citet{basu2021no}, for each $s$, we can bound w.p. $1$
    $$\Gamma_{s,t} \leq 4\sqrt{\frac{\lambda_1(\Sigma_{0})\left(1+\tfrac{\lambda_{1}(\Sigma_q)(1+  \tfrac{\sigma^2 / \eta}{\lambda_{1}(\Sigma_0)})}{\lambda_{1}(\Sigma_0) + \sigma^2 / \eta +  \sqrt{s}\lambda_{1}(\Sigma_q)}\right)}{\log\left(1+\tfrac{\lambda_1(\Sigma_0)}{\sigma^2}\left(1+\tfrac{\lambda_{1}(\Sigma_q)(1+  \tfrac{\sigma^2 / \eta}{\lambda_{1}(\Sigma_0)})}{\lambda_{1}(\Sigma_0) + \sigma^2 / \eta +  \sqrt{s}\lambda_{1}(\Sigma_q)}\right)\right)}\log(4|\ca{A}|/\delta)}.$$
    This is true by using the upper bounds on $\sigma_{max}^2(\hat{\Sigma}_{s,t})$ in Lemma~\ref{lemm:concentration}, and  because the function $\sqrt{x/\log(1+ax)}$ for $a>0$ increases with $x$. 
    Therefore, we have the bounds $\Gamma_{s,t} \leq \Gamma_s$ w.p. $1$ for all $s$ and $t$ by using appropriate $s$, and by setting $s=0$ we obtain $\Gamma$.  
    
    For a matrix $A\in\R^{d\times d}$, let $\lambda_{\ell}(A)$ denote its $\ell$-th largest eigenvalue for $\ell\in[d]$.
    By \cref{lem:infoBd} the following holds for any  $\delta > 0$
    \begingroup
    \allowdisplaybreaks
    \begin{align*}
      &\BSR(m,n) \leq  \Gamma\sqrt{\frac{m}{n}\;\Info(\theta_*;\tau_{1:m})}+\sum_{s=1}^m\Gamma_s\sqrt{\frac{\Info(\mu_s;\tau_s|\theta_*,\tau_{1:s-1})}{n}} +\sum_{s=1}^m\sum_{t=1}^n\frac{\E\beta_{s,t}}{n}\\
      &{\color{red} \leq} 4
      \sqrt{C_1\log(4\mid \ca{A}\mid /\delta)} 
      \sqrt{\frac{m}{n} \tfrac{d}{2}\log\left(1+  \frac{mn \lambda_1(\Sigma_q)}{n\lambda_{d}(\Sigma_0) + \sigma^2}\right)}
      \\
      &+ \sum_{s=1}^{m} 4\sqrt{\frac{\lambda_1(\Sigma_{0}){\color{cyan}\left(1+\tfrac{\lambda_{1}(\Sigma_q)(1+  \tfrac{\sigma^2 / \eta}{\lambda_{1}(\Sigma_0)})}{\lambda_{1}(\Sigma_0) + \sigma^2 / \eta +  \sqrt{s}\lambda_{1}(\Sigma_q)}\right)}}{\log\left(1+\tfrac{\lambda_1(\Sigma_0)}{\sigma^2}{\color{blue}\left(1+\tfrac{\lambda_{1}(\Sigma_q)(1+  \tfrac{\sigma^2 / \eta}{\lambda_{1}(\Sigma_0)})}{\lambda_{1}(\Sigma_0) + \sigma^2 / \eta +  \sqrt{s}\lambda_{1}(\Sigma_q)}\right)}\right)}\log(4|\ca{A}|/\delta)} \sqrt{\tfrac{1}{n}\tfrac{d}{2}\log\left(1 + n\frac{\lambda_1(\Sigma_0)}{\sigma^2}\right) }
       \tag{$\Info$ bounds from Lemma 4 \citet{basu2021no}}
      \\
      & + \frac{\delta}{2n}\max_{a\in\A}\|a\|_2 \sum_{s=1}^{m}\sum_{t=1}^{n} \E\E_{s,t}[\|\mu_s \|_2]\\
      &{\color{red} \leq} 4
      \sqrt{
      C_1\log(4|\ca{A}|/\delta)} 
      \sqrt{\frac{m}{n} \tfrac{d}{2}\log\left(1+  \frac{mn \lambda_1(\Sigma_q)}{n\lambda_{d}(\Sigma_0) + \sigma^2}\right)}
      \\
      &+  \left(m+ {\color{cyan} \tfrac{1}{2 \lambda_{1}(\Sigma_0)} \sum_{s=1}^{m} \tfrac{\lambda_{1}(\Sigma_q)( \lambda_{1}(\Sigma_0)+  \sigma^2/ \eta )}{\lambda_{1}(\Sigma_0) + \sigma^2/\eta  +  \sqrt{s}\lambda_{1}(\Sigma_q)}}\right) \times 
      \left(4 \sqrt{C_2\log(4|\ca{A}|/\delta)} \sqrt{\tfrac{1}{n} \tfrac{d}{2}\log\left(1 + n\frac{\lambda_1(\Sigma_0)}{\sigma^2}\right) } \right) \tag{Remove {\color{blue} highlighted}, {\color{cyan} $\sqrt{1+x}\leq 1+ x/2$ for all $x\geq 1$}}\\
      & + \frac{\delta}{2n}\max_{a\in\A}\|a\|_2 \sqrt{mn(\|\mu_{*}\|^2_2 + \tr(\Sigma_q + \Sigma_0))}\tag{\cref{eq:normMu}}
      \\
      &{\color{red} \leq} 4
      \sqrt{C_1\log(4|\ca{A}|/\delta)} 
      \sqrt{\frac{m}{n} \tfrac{d}{2}\log\left(1+  \frac{mn \lambda_1(\Sigma_q)}{n\lambda_{d}(\Sigma_0) + \sigma^2}\right)}
      \\
      &+  \left(m+ {\color{cyan}  (1+  \tfrac{\sigma^2/\eta }{\lambda_{1}(\Sigma_0)})\sqrt{m}}\right)\times 
      \left( 4\sqrt{C_2\log(4|\ca{A}|/\delta)} \sqrt{\tfrac{1}{n} \tfrac{d}{2}\log\left(1 + n\frac{\lambda_1(\Sigma_0)}{\sigma^2}\right) } \right) \tag{{\color{cyan} Integral}}\\
      & + \frac{\delta}{2}\max_{a\in\A}\|a\|_2 \sqrt{\tfrac{m}{n}(\|\mu_{*}\|^2_2 + \tr(\Sigma_q + \Sigma_0))} 
    \end{align*}
    \endgroup
    where 
    \begin{align*}
        C_1 &= \frac{\lambda_1(\Sigma_{q}) + \lambda_1(\Sigma_{0})}{\log(1+(\lambda_1(\Sigma_q)+\lambda_1(\Sigma_{0}))/\sigma^2 )}\\
        C_2 &=\frac{\lambda_1(\Sigma_{0})}{\log\left(1+\tfrac{\lambda_1(\Sigma_0)}{\sigma^2}\right)}
        \;.
    \end{align*}
    The first inequality substitutes $\Info_{s,t}$ terms by the appropriate bounds from Lemma 4 \citet{basu2021no}. The second inequality first removes the part highlighted in blue (which is positive) inside the logarithm, and then uses the fact that $\sqrt{1+x}\leq 1+ x/2$ for all $x\geq 1$. We also use 
    \begin{align}
        \E[\|\mu_s \|_2] = \sqrt{\|\mu_{*}\|^2_2 + \tr(\Sigma_q + \Sigma_0)}\label{eq:normMu}\;.
    \end{align}
    The final inequality replaces the summation by an integral over $s$ and derives the closed form.

\end{proof}

\section{Method of Moments for The Bernoulli bandits}\label{app:MoM-Ber}

Based on the procedure explained in \cref{sec:Examp-Bern}, \expl samples arm 1 in the first $t_0$ rounds of first $m_0/K$ tasks, and arm 2 in the next $m_0/K$ tasks similarly, and so on for arm 3, 4, up to $K$. In other words, \expl samples arm $a\in[K]$ in the first $t_0$ rounds of the $a$'th batch of size $m_0/K$ tasks. Let $X_s$ denote the cumulative reward collected in the first $t_0$ rounds of task $s$. Then, the random variables $X_1, \cdots , X_{m_0/K}$ are i.i.d.\ draws from a Beta-Binomial distribution with parameters $(\alpha^*_1, \beta^*_1, t_0)$, where $t_0$ denotes the number of trials of the binomial component.

For arm 1, we can retrieve $\alpha^*_1, \beta^*_1$ based on the following equations stating the first and second moments of $X_s$,
\begin{align}
\label{eq:BernEst}
\E[X_s]=\frac{t_0\alpha^*_1}{\alpha^*_1+\beta^*_1},
\quad  \E[X_s^2]=\frac{t_0\alpha^*_1(t_0(1+\alpha^*_1)+\beta^*_1)}{(\alpha^*_1+\beta^*_1)(1+\alpha^*_1+\beta^*_1)}\;.
\end{align}
where we assume $t_0\geq 2$. Therefore, we can estimate the prior using estimates of $\E[X_s], \E[X_s^2]$, via the method of moments \citep{tripathi1994estimation}. In particular,
\begin{align}
    \hat{\alpha}^*_{1} &=\frac{t_0\E^2[X_s]-\E[X_s^2]\E[X_s]}{t_0(\E[X_s^2]-\E^2[X_s]-\E[X_s])+\E^2[X_s]}
    \\
    \widehat{\beta}^*_{1} &=\frac{(t_0-E[X_s])(\E[X_s]t_0-\E[X_s^2])}{t_0(\E[X_s^2]-\E^2[X_s]-\E[X_s])+\E^2[X_s]}
\label{eq:MoM}
\end{align}

For the rest of the arms, we can use a similar technique.

\section{Frequentist Analysis}\label{sec:frqPfs}

\todob{*** To be discussed at the next meeting***

I strongly believe that this section can be significantly simplified, because we do not need to go into the details of Simchowitz at all. Think of it as follows. For any task $\mu$, the cumulative regret of the algorithm with the misspecified prior decomposes into two components:

1) Cumulative regret of the algorithm that knows the prior. Let us call it $A(\mu)$.

2) Difference in cumulative rewards of the algorithm with the correct and misspecified priors. Let us call it $B(\mu)$.

By cumulative-to-simple regret reduction, we know that $A(\mu) + B(\mu)$ cumulative regret translates to $(A(\mu) + B(\mu)) / n$ simple regret, when the best arm is chosen proportionally to the number of pulls in $n$ rounds. By existing gap-free analyses of TS, we have a bound on $E_{\mu \sim P_*}[A(\mu)]$. By Simchowitz, we have a bound  on $E_{\mu \sim P_*}[B(\mu)]$.}

In this section, we provide the results needed in \cref{sec:frqAny} from \citet{simchowitz2021bayesian}. The rearrangement of these results here is helpful to understand the analysis of that paper and the way we use them. We let $P_{\theta, \alg(\theta')}(\mu, \tau_n)$ be the joint law over the task mean $\mu$ and the full trajectory of the task, $\tau_n$, when the prior parameter is $\theta$ while $\alg$ is initialized with prior parameter $\theta'$.
Note that since the posterior of $\mu$ given the trajectory of the algorithm, $\tau_n$, is conditionally independent of the algorithm, 
we can use $P_{\theta}(\mu|\tau_t):=P_{\theta,\alg(\theta')}(\mu|\tau_t)$ for any $\mu$ given the trajectory at round $t$.

We define the \emph{Monte Carlo} family of algorithms as follows.

\begin{definition}[Monte-Carlo algorithm~\citealt{simchowitz2021bayesian}]\label{defi:MC}
Given $\gamma>0$, any base algorithm \alg instantiated with the prior parameter $\theta$ is $\gamma$-Monte Carlo if for any $\theta'$ and trajectory $\tau_t,\;t\geq 1$, we have
\begin{align*}
    \tv\left(P_{\alg(\theta)}(A_t|\tau_t)\right.&\left.\parr P_{\alg(\theta')}(A_t|\tau_t)\right)\leq \gamma~\tv\left(P_{\theta}(\mu|\tau_t)\parr P_{\theta'}(\mu|\tau_t)\right)\;.
\end{align*}
where $P_{\alg(\theta)}(A_t|\tau_t)$ is the probability of choosing arm $A_t$ at round $t$ by $\alg(\theta)$, given the trajectory of the task up to the beginning of round $t$, and $P_{\theta}(\mu|\tau_t)$ is the posterior of the mean reward, given the trajectory up to round $t-1$ when the algorithm is initialized with prior $\theta$. 
\end{definition}

An important instance of Monte Carlo algorithms is TS, which is 1-Monte Carlo~\citep{simchowitz2021bayesian}. 

First, we recite the following proposition regarding the TV distance of the trajectories under different prior initialization. The TV distance between trajectories of the algorithm with correct prior and the same algorithm with an incorrect prior has the following upper bound, which results from \cref{defi:MC}.

\begin{proposition}[TV Distance of Two Trajectories, Proposition 3.4., \citet{simchowitz2021bayesian}]\label{prop:TV}
Let $\alg$ be a $\gamma$-Monte Carlo algorithm for $n\in \mathbb{N}$ rounds. Then
\begin{align*}
    \tv\left(P_{\theta_*,\alg(\theta_*)}(\mu,\tau_n)\parr P_{\theta_*,\alg(\theta_*')}(\mu,\tau_n)\right)\leq 2\gamma n\tv(P_{ \theta_*}\parr P_{\theta_*'})
\end{align*}\jtodo[inline]{the n factor comes in.}
holds for any $\mu$ and $\tau_n$.
\end{proposition}
See Proposition 3.4. of \citet{simchowitz2021bayesian} for the proof.

We also need the following definition to state the next lemma for bounding the regret of our algorithm.

\begin{definition}[Upper Tail Bound, Definition B.2., \citet{simchowitz2021bayesian}]\label{defi:UTB}
Let $X$ be a non-negative random variable defined on probability space $(\Omega,\ca{F})$ with probability law $P$ and expectation $\E[X]<\infty$, and $Y\in[0,1]$ also be another random variable defined on the same probability space, then Upper Tail Bound is defined as 
\begin{align}
    \Psi_X(p)&:=\frac{1}{p}\sup_{Y}\E[XY]\nonumber\\&\text{s.t.}~ \E[Y]\leq p\nonumber
\end{align}
For $p> 1$ we extend the definition by setting $\Psi_X(p)=\E[X]$.
\end{definition}

\begin{lemma}[Relative Regret]\label{lem:MR}
    Let $\Alg= \alg(\theta_*)$ be an algorithm with prior parameter $\theta_*$ and $\Alg'=\alg(\theta_*')$ be the same with different prior parameter $\theta_*'$. Then the difference between their simple regrets in a task with $n$ rounds coming from prior $P_{\theta_*}$ is bounded as follows 
    \begin{align*}
        \E_{\mu\sim P_{\theta_*}}\E[\mu(\Ahat_{\Alg})-\mu(\Ahat_{\Alg'})]\leq \delta\Psi_{\theta_*}(\delta)
    \end{align*}
    when $\delta=\tv(P_{\theta_*,\Alg}\parr P_{\theta_*,\Alg'})$ and $\Psi_{\theta_*}(p):=\Psi_{\diam(\mu)}(p)$ for $\mu\sim P_{\theta_*}$.
\end{lemma}
\begin{proof}
    Considering $P_{\theta_*,\alg}(\mu,\tau_n)$ and $P_{\theta_*,\alg'}(\mu,\tau_n)$ as $\mu$ is independent of \alg prior given the trajectory then $P_{\tau}:=P_{\theta_*,\alg}(\mu)=P_{\theta_*,\alg'}(\mu)=:P'_{\tau}$. Now by \cref{lem:CoupleTransport} we know there exists a coupling $Q(\mu,\tau_n,\tau'_n)$ such that 
    \begin{align*}
        Q(\mu,\tau_n)=P_{\tau},\quad Q(\mu,\tau'_n)=P'_{\tau},\quad Q[\tau_n\neq \tau'_n]=\tv(P_{\tau},P'_{\tau})=:\delta
    \end{align*}
    Now let $\E_Q$ be the corresponding expectation then
    \begin{align*}
        \E\left[\mu(\Ahat_{\Alg})-\mu(\Ahat_{\Alg'})\right]&\leq \E_Q\left[\diam(\mu)\I{\Ahat_{\Alg}\neq \Ahat_{\Alg'}}\right]\\&\leq
        \E_Q\left[\diam(\mu)\I{\tau_n\neq \tau'_n}\right]\\&\leq
        \delta\Psi_{\theta_*}(\delta)
    \end{align*}
    \jnote[inline]{this is where the $n$ factor comes in by looking at the total trajectory, i.e., $\E_Q\I{\Ahat_{\Alg}\neq \Ahat_{\Alg'}} \leq \E_Q\I{\tau_n\neq \tau'_n}$.
    
    Let $f(\tau)\in\R^K$ be the vector of arms play proportions in trajectory $\tau$. We know 
    \[\E_Q\I{f(\tau_n)\neq f(\tau'_n)} \leq \E_Q\I{\tau_n\neq \tau'_n},\] 
    so if we can bound the former, we might be able to avoid factor $n$.
    
    Note that given $f(\tau_n)$ all $\tau_n$ corresponding to $f(\tau)$ are equally likely. Given a frequency vector $f_n$ there are $\frac{n!}{\prod_{i=1}^K f_n(i)!}$ different trajectories corresponding to it, lets call the set corresponding to these $\ca{T}(f(\tau_n))$.
    \begin{align*}
        \E_Q\I{f_n= f'_n} = \E_Q\I{|\ca{T}(f_n)\cap \ca{T}(f'_n)|\geq 1}\geq 
    \end{align*}}
    where we used the \cref{defi:UTB} in the last inequality and the fact that $\E_Q\left[\I{\tau_n\neq \tau'_n}\right]=Q\left[\tau_n\neq \tau'_n\right]=\delta$ by definition of $Q$.
\end{proof}

The next result is a generic bound on the relative regret of a Monte-Carlo algorithm compared to an oracle which knows the prior.

\thmfrqMSR*

\begin{proof}[Proof of \cref{thm:frqMSR}]
    By \cref{lem:MR} we know 
    $
        \E[\mu(\Ahat_{\Alg^*})-\mu(\Ahat_{\Alg})]\leq \delta\Psi_{\theta_*}(\delta)\nonumber
    $
    where $\delta=\tv(P_{\theta_*,\Alg^*}\parr P_{\theta_*,\Alg})$. Then by \cref{prop:TV} we know $\delta\leq 2\gamma n\epsilon$. Now since $p\mapsto p\Psi_{\theta_*}(p)$ is non-decreasng in $p$ (Lemma B.5 from \citep{simchowitz2021bayesian}) we get
    \begin{align}
        \E[\mu(\Ahat_{\Alg^*})-\mu(\Ahat_{\Alg})]\leq 2\gamma n\epsilon\Psi_{\theta_*}(2\gamma n\epsilon)\nonumber
    \end{align}
    Finally, by \cref{lem:UTBconds} we get  $\Psi_{\theta_*}(p)\leq B$ if $P_{\theta_*}$ satisfies $P_{\theta_*}(\diam(\mu)\leq B)=1$, whic concludes the first part of the proof.
    
    For the second part we make sure $2\gamma n\epsilon\in[0,1]$, by using $\min(1,2\gamma n\epsilon)$, then again \cref{lem:UTBconds} gives
    \[\Psi_{\theta_*}(2\gamma n\epsilon)\leq \diam(\E_{\theta_*}[\mu])+\sigma_0\left(8+5\sqrt{\log \frac{|\A|}{\min(1,2\gamma n\epsilon)}}\right)\]
\end{proof}

\corfrqSR*

\begin{proof}[Proof of \cref{cor:frqSR}]
    The frequentist meta simple regret decomposes in two terms.
    \begin{align*}
        \SR(m,n,P_{\theta_*})&=\sum_{s=1}^m \E_{\mu_s\sim P_{\theta_*}} [\mu_s(A_s^*)-\mu_s(\Ahat_{\alg_s})]
        \\
        &=\sum_{s=1}^m \E_{\mu_s\sim P_{\theta_*}} [\mu_s(A_s^*)-\mu_s(\Ahat_{\alg^*})]+\E[\mu_s(\Ahat_{\alg^*})-\mu_s(\Ahat_{\alg_s})]
    \end{align*}
    where $\alg^*$ is the oracle algorithm that is initialized with the correct prior $P_{\theta_*}$. Now by \cref{prop:cum2sim} and the properties of $\gamma$-Monte Carlo algorithm, \alg, we can bound the  the first term by $O(m\sqrt{|\A|/n})$. This is because per-task cumulative regret of Monte Carlo algorithm is $O(\sqrt{n|\A|}$, e.g., for TS which we use this holds \citep{agrawal2013further}.
    
    The second term is bounded by $\sum_{s=1}^m 2n\gamma\epsilon_s B$ based on \cref{thm:frqMSR}. Now, if $\epsilon_s=O(1/\sqrt{s})$ we know $O(\sum_{s=1}^m 2n\gamma B/\sqrt{s} )=O(\sqrt{m}n\gamma B$.
    
    For a sub-Gaussian prior, we can use the bound for the second term from \cref{thm:frqMSR} to get the following performance guarantee similarly
    \begin{align*}
        &\SR(m,n,P_{\theta_*})
        \\&=O\Bigg(2\sqrt{m}\gamma n\diam(\E_{\theta_*}[\mu])+\sigma_0\sum_{s=1}^m\left(8+5\sqrt{\log\frac{|\A|}{\min(1,2\gamma n /\sqrt{s})}}\right) 
        \\&\hspace{2in}+m \sqrt{|\A|/n}+m_0 B\Bigg)
    \end{align*}
\end{proof}

\subsection{Technical Tools}
In this section we recite some technical tools from \citet{simchowitz2021bayesian} that are used in our proofs.
\begin{lemma}[Coupled Transport Form, Lemma B.4., \citet{simchowitz2021bayesian}]\label{lem:CoupleTransport}
    Let $P$ and $P'$
    be joint distributions over random variables
    $(X, Y)$ with coinciding marginals $P(X) = P'(X)$
    in the first variable. Then there exists a distribution $Q(X, Y, Y')$
    whose marginals satisfy $Q(X, Y ) = P(X, Y )$ and $Q(X, Y'
    ) = P'(X, Y )$, and
    for which we have
    \[
    \tv(P(X, Y ) \parr P'(X, Y)) = Q[Y\neq Y']\].
\end{lemma}

\begin{definition}[Tail Conditions]\label{defi:Tails}
Let $\bar{\mu}_{\theta}=\E_{\theta}[\mu]$. We say $P_{\theta}$ is 
\begin{enumerate}[label=(\roman*)]
    \item $B$-bounded if $P_{\theta_*}(\diam(\mu)\leq B)=1$
    \item Coordinate-wise $\sigma^2$-sub-Gaussian if for all $a\in\A$,
        \[P_{\theta}(|\mu_a-\bar{\mu}_{\theta}|\geq t)\leq 2\exp(\frac{-t^2}{2\sigma^2})\]
    \item Coordinate-wise $(\sigma^2, v)$-sub-Gamma if for all $a \in \A$,
        \[P_{\theta}(|\mu_a-\bar{\mu}_{\theta}|\geq t)\leq 2\max\{\exp(\frac{-t^2}{2\sigma^2}),\exp(\frac{-t}{2v})\}\]
\end{enumerate}

\end{definition}
\begin{lemma}[Upper Tail Bound under Tail Conditions, Lemma B.6., \citet{simchowitz2021bayesian}]\label{lem:UTBconds}
Let $\bar{\mu}_{\theta}=\E_{\theta}[\mu]$. Then for any $p\in[0,1]$
\begin{enumerate}[label=(\roman*)]
    \item If $P_{\theta}$ is $B$ bounded, then $\Psi_{\theta}(P)\leq B$ for all $p$.
    \item If $P_{\theta}$ is coordinate-wise $\sigma^2$-sub-Gaussian and $\A$ is finite, then
    \[\Psi_{\theta}(P)\leq \diam(\bar{\mu}_{\theta})+\sigma\left(8+5\sqrt{\log \frac{2|\A|}{p}}\right)\]
    \item if $P_{\theta}$ is coordinate-wise $(\sigma^2,v)$-sub-Gamma and $\A$ is finite, then
    \[\Psi_{\theta}(P)\leq \diam(\bar{\mu}_{\theta})+\sigma\left(8+5\sqrt{\log \frac{2|\A|}{p}}\right)+v\left(11+7\log \frac{2|\A|}{p}\right)\]
\end{enumerate}
We can extend these for $p\geq 1$ by replacing $p\gets \min(1,p)$.
    
\end{lemma}

\begin{lemma}[Pinsker's Inequality]\label{lem:Pins}
    If $P$ and $Q$ are two probability distributions on a measurable space $(X,\Sigma)$, then
    \begin{align*}
        \tv(P\parr Q)\leq \sqrt{\frac{1}{2}\KL(P\parr Q)}
    \end{align*}
\end{lemma}

\begin{lemma}[Gaussian KL-divergence]\label{lem:KLG}
If $P=\Ngaus(\theta, \Sigma)$ and $\hat{P}=\Ngaus(\hat{\theta}, \hat{\Sigma})$ then
    \begin{align}
        \KL(P\parr\hat{P})=\frac{1}{2}
        \left(
        \tr(\Sigma^{-1/2}\hat{\Sigma}\Sigma^{-1/2}-I)-\log\det(\Sigma^{-1/2}\hat{\Sigma}\Sigma^{-1/2})+ \vnormf{\Sigma^{-1/2}(\hat{\theta}-\theta)}_2^2
        \right)\label{eq:KLG}
    \end{align}
\end{lemma}
The proof is a standard result in statistics.

\subsection{Lower Bound}\label{app:LB}

\begin{theorem}[Lower Bound]
    Consider any $\gamma$-shot TS algorithm $\TS_{\gamma}(\cdot)$ for $\gamma\in\mathbb{N}$ and a task with prior $P_{\theta}$ over bounded mean rewards
    $\mu \in [0, 1]^{|\A|}$ with ${|\A|}=n\lceil\frac{c_0}{\eta}\rceil$. Then there exists universal constant $c_0$ for a fixed $\eta\in(0,1/4)$ such that for any horizon $n \gg \frac{c_0}{\eta}$ and error $\epsilon\leq \frac{\eta}{c_0\gamma n}$, there exists prior $P_{\theta'}$ with $\tv(P_{\theta}\parr P_{\theta'}) =\epsilon$ and
    \begin{align}
        \E_{\mu\sim P_{\theta}}\E[\mu(\Ahat_{\TS_{\gamma}(\theta)})-\mu(\Ahat_{\TS_{\gamma}(\theta')})] \geq  (\frac{1}{2}-\eta) \gamma n\epsilon\nonumber
    \end{align}
\end{theorem}

\begin{proof}
    With the assumptions here, Theorem D.1 from \citet{simchowitz2021bayesian} states that
    \begin{align*}
        R(n,P_{\theta})-R(n,P_{\theta'})\geq (\frac{1}{2}-\eta) \gamma n^2\epsilon
    \end{align*}
    for $\TS_{\gamma}(\theta)$ and $\TS_{\gamma}(\theta')$. Now by \cref{prop:cum2sim} and linearity of expectation we get the result as we divide the RHS with $n$.
\end{proof}

\begin{lemma}\label{lem:expX}
    Let $X$ be a random variable supported on $\{b_1,\cdots,b_K\}\subset\R$ with $b_i\leq 1$ and $p_i:=\P(X=b_i)$ for all $i$. Then
    \begin{align*}
        \E[\exp(X)]\leq \exp\left(\sum_{i=1}^Kp_i(b_i+b_i^2)\right)
    \end{align*}
\end{lemma}
\begin{proof}
    As $e^t\leq 1+t+t^2$ for all $t\leq 1$, we have
    \begin{align*}
        \E[\exp(X)]\leq \E[1+X+X^2]=1+\sum_{i=1}^Kq_i(b_i+b_i^2)
    \end{align*}
    Then we can get the result noting that $1+t\leq e^t$ for any $t\in\R$.
\end{proof}

\subsection{Proofs of Frequentist Bernoulli}\label{sec:freBern-pfs}

We first prove the following result on the relative simple regret of a Monte-Carlo algorithm compared to an oracle algorithm which knows the prior. This algorithm uses the method of moments estimator of \cref{eq:BernEst}. 

\begin{corollary}[Relative Per-task Simple Regret, Bernoulli Bandits]
\label{cor:frqMSR-Bern}
    Let $\alg$ be an $\gamma$-Monte Carlo algorithm. 
    Under the setting of \cref{sec:Examp-Bern}, let $\hat{\beta}_*$ be the estimated prior parameters based on \cref{eq:BernEst},
    and $\Alg=\alg(\theta_*)$ and $\Alg'=\alg(\hat{\theta}_*)$ be oracle \alg and \alg instantiated by the estimated prior in a task after $m_0$ exploration tasks, respectively. Then for any $\epsilon$ there is a constant $C$ such that if 
    $\MBern$,
    we know
    \begin{align}
        \E[\mu(\Ahat_{\Alg})-\mu(\Ahat_{\Alg'})]&\leq 2\gamma n\epsilon\nonumber
    \end{align}
    with probability at least $1-\delta$.
\end{corollary}

\begin{proof}[Proof of \cref{cor:frqSR-Bern}]
    By Theorem 4.1 from \citet{simchowitz2021bayesian}, we know if $\MBern$, then $\tv(P_{\theta_*}\parr P_{\hat{\theta}_*})\leq \epsilon$ with probability $1-\delta$. Now, as Bernoulli rewards with beta Priors are bounded by 1, then by \cref{thm:frqMSR} we get the result replacing $B$ with 1.
\end{proof}

\corBernSR*

\begin{proof}[Proof of \cref{cor:frqSR-Bern}]
    The frequentist meta simple regret decomposes in three terms.
    As Bernoulli is a $1$-bounded distribution, the $m_0$ term is an upper bound on the simple regret of the exploration tasks. Then for the rest of the tasks, we can use the following decomposition
    \begin{align*}
        \SR(m,n,P_{\theta_*})&=\sum_{s=1}^m \E_{\mu_s\sim P_{\theta_*}} [\mu_s(A_s^*)-\mu_s(\Ahat_{\alg_s})]
        \\
        &=\sum_{s=1}^m \E_{\mu_s\sim P_{\theta_*}} [\mu_s(A_s^*)-\mu_s(\Ahat_{\alg^*})]+\E[\mu_s(\Ahat_{\alg^*})-\mu_s(\Ahat_{\alg_s})]
    \end{align*}
    where $\alg^*$ is the oracle algorithm. Now by \cref{prop:cum2sim} and a problem-independent cumulative regret bound of TS \citep{agrawal2013further}, \citep[Theorem 36.1]{lattimore-Bandit}, we can bound the first term by $O(m\sqrt{|\A|\log(n)/n})$. The second term is bounded based on \cref{cor:frqMSR-Bern} by $\sum_{s=1}^m 2n\gamma\epsilon=2mn\gamma\epsilon$ for a $\gamma$-Monte Carlo algorithm. For TS $\gamma=1$.
\end{proof}

\subsection{Proofs of Frequentist Linear Bandits}\label{app:fLinGausPfs}

In this section we extend the results of \citet{simchowitz2021bayesian} for meta-learning to linear bandits. First note the following result on the KL-divergence of two Gaussian random variables corresponding to the prior and the estimated prior. 

\begin{lemma}[Gaussian KL-divergence]\label{lem:KLG-2}
If $P=\Ngaus(\theta, \sigma_0^2I_d)$ and $\hat{P}=\Ngaus(\hat{\theta}, \sigma_0^2I_d)$ then
    \begin{align}
        \KL(P\parr\hat{P})=\frac{1}{2\sigma_0^2}
         \vnormf{\hat{\theta}-\theta}_2^2\label{eq:KLG-2}
    \end{align}
\end{lemma}
This is a special case of \cref{lem:KLG}. \cref{lem:KLG-2} along with Pinsker's inequality (\cref{lem:Pins}), implies that we need to design an estimator such that the RHS of \cref{eq:KLG-2} is bounded. 

\begin{lemma}\label{lem:LinGaus-est-norm}
    \jtodo{sub-Gaussian}
    Consider a Gaussian prior $P_*=\Ngaus(\theta_*,\sigma_0^2I_d)$ and consider the setting of \cref{sec:LinGaus}, then
    \begin{align*}
        (\mu_1,a_1,y_{1,1}),&\cdots,(\mu_1,a_d,y_{1,d})\\
        (\mu_2,a_1,y_{2,1}),&\cdots,(\mu_1,a_d,y_{2,d})\\
        &\vdots\\
        (\mu_{m_0},a_1,y_{{m_0},1}),&\cdots,(\mu_1,a_d,y_{{m_0},d})
    \end{align*}
    for some ${m_0}\leq m$ be random variables such that
    \begin{align*}
        \mu_s\overset{\text{i.i.d}}{\sim} P_*,\quad y_{s,i}|(\mu_s,a_i)\overset{\text{i.i.d}}{\sim} \Ngaus(a_i\T\mu_s,\sigma^2)
    \end{align*}
    
    and finally define
    \begin{align*}
        \hatTheta\;.
    \end{align*}
    where again $\VmDef$ is the outer product of the basis.
    \todob{*** To be discussed at the next meeting***
    
    The estimator of $\theta_*$ may be invalid. The reason is that we analyze a regression problem where the first $d$ observations in each task are additionally correlated because of the task. The MLE of $\theta_*$ in this setting is the same as in Basu. There are two ways out of this:
    
    1) Use the estimator of Basu.
    
    2) De-correlate the observations. One approach would be that only the first round in each task $s$ is exploratory, with pulled arm $a_i$ for $i = ((s - 1) \mod d) + 1$.
    
    In addition to the above, it is unclear how the actions $a_1, \dots, a_d$ are chosen in the current proof. This should be some basis. No need for optimal designs.} 
    \jtodo{
    1) sigma q = inf
    
    2) Solve OLS for d samples in each task, then OLS estimate is Gauss, then it is centered at theta, better bound}
    
    Then for any $\delta\in(2e^{-d},1)$
    \begin{align*}
        \vnormf{\theta-\hat{\theta}}_2\leq \tnorm
    \end{align*}
    with probability at least $1-\delta$.
\end{lemma}
\begin{proof}
    We can write $y_{s,i}=a_i\T\theta_*+a_i\T\xi_{s,2}+\xi_{s,1}$ where $\xi_{s,1}\sim\Ngaus(0,\sigma^2)$ and $\xi_{s,2}\sim\Ngaus(0,\sigma_0^2I_d)$ are independent. Now by an Ordinary Least Squares estimator we constructs the estimator as follows 
    \begin{align*}
        \hat{\theta}_*=\Vmi\sum_{s=1}^{m_0}\sum_{i=1}^da_iy_{s,i}
    \end{align*}
    and 
    \begin{align*}
        \E[\hat{\theta}_*]&=\E[\Vmi\sum_{s=1}^{m_0}\sum_{i=1}^d a_i( a_i\T\theta_*+a_i\T\xi_{s,2}+\xi_{s,1})]
        \\&=\E[\Vmi(V_{m_0}\theta_*+\sum_{s,i} a_ia_i\T\xi_{s,2}+ a_s\xi_{s,1})]
        \\&=\theta_*+\sum_{s,i}\E[\Vmi a_ia_i\T\xi_{s,2}]+\sum_{s,i}\E[\Vmi a_i\xi_{s,1}]
        \\&=\theta_*+\sum_{s,i}\Vmi a_ia_i\T\E[\xi_{s,2}]+\sum_{s,i}\Vmi a_i\E[\xi_{s,1}]=\theta_*
    \end{align*}
    
    Now we bound $\vnormf{\hat{\theta}_*-\theta_*}_2$ as follows
    
    \begin{align*}
        \vnormf{\hat{\theta}_*-\theta_*}_2&=\vnormf{\Vmi V_{m_0}(\hat{\theta}_*-\theta_*)}_2\\
        &=\vnormf{\Vmi\Big(\sum_{s,i}a_i(a_i\T\theta_*+a_i\T\xi_{s,2}+\xi_{s,1})-V_{m_0}\theta_*\Big)}_2\\
        &=\vnormf{\Vmi\sum_{s,i}a_i(a_i\T\xi_{s,2}+\xi_{s,1})}_2\\
        &\leq \vnormf{\Vmi}_2\vnormf{\sum_{s,i}a_i(a_i\T\xi_{s,2}+\xi_{s,1})}_2\\
    \end{align*} 
    
    Now note that $\vnormf{\Vmi}_2$ is the square root of the largest eigenvalue of $\Vmi\Vmi$ which since $V_{m_0}$ is positive definite (by assumption), it equals $\lambda_d^{-1}(V_{m_0})=\frac{1}{{m_0}}\lambda_d^{-1}(\sum_{i=1}^d a_i a_i\T)$. Also, $Z_{s,i} = a_i(a_i\T\xi_{s,2}+\xi_{s,1})$ is a vector with independent $\sigma_i:=\left(\sqrt{\sigma_0^2 \vnormf{a_i}_2^4+ \sigma^2\vnormf{a_i}_2^2}\right)$-sub-Gaussian coordinates (by independence of $\xi_{s,1}$ and $\xi_{s,2}$). We know $Z_{s,i}$'s are independent since the chosen arms are fixed. Then $Z=\sum_{s=1}^{m_0}\sum_{i=1}^d Z_{s,i}\in\mathbb{R}^d$ is a vector with $(\sqrt{{m_0}\sum_{i=1}^d \sigma_i^2})$-sub-Gaussian coordinates and we know  
    \begin{align*}
        \vnormf{\sum_{s,i}a_i(a_i\T\xi_{s,2}+\xi_{s,1})}_2=\vnormf{Z}_2=\sqrt{\sum_{l=1}^d Z_l^2}
    \end{align*}
    where $Z_l$ is the $l$'th coordinate of $Z$. Therefore, by Bernstein’s inequality (Theorem 2.8.1 of
    \citet{Vershynin-HDP-2019}) we know
    \begin{align*}
        \P(\sum_{l=1}^d Z_l^2\geq t)\leq 2\exp(-\min\{\frac{t^2}{d{m_0}\sum_{i=1}^d \sigma_i^2}, \frac{t}{\sqrt{{m_0}\sum_{i=1}^d \sigma_i^2}}\}) 
    \end{align*}
    Thus $\vnormf{Z}_2\leq (d {m_0} \sum_{i=1}^d\sigma_i^2\log(2/\delta))^{1/4}$ with probability at least $1-2\exp(-\min\{\log(2/\delta), \sqrt{d \log(2/\delta)}\})$ which is $1-\delta$ if $\delta\geq 2\exp(-d)$ and 1-$\exp(-\sqrt{d\log(2/\delta)})$ otherwise. Therefore
    \begin{align*}
        \vnormf{\Vmi}_2\vnormf{\sum_{s,i}a_i(a_i\T\xi_{s,2}+\xi_{s,1})}_2 &\leq \frac{1}{{m_0}}\lambda_d^{-1}(\sum_{i=1}^d a_i a_i\T) (d {m_0} \log(2/\delta) \sum_{i=1}^d\sigma_i^2)^{1/4}\\
        &= \tnorm
    \end{align*}
    with the probability discussed above.
\end{proof}

Next, we prove the following for \expl of \cref{eq:LinGaus-theta-hat}.

\thmlinBanf*
\begin{proof}[Proof of \cref{thm:linBan-f-est}]
    By Pinsker's inequality (\cref{lem:Pins}) and \cref{lem:KLG-2} we know
    \begin{align}
        \tv(P_{\theta_*}\parr P_{\hat{\theta}_*})&\leq \sqrt{\frac{1}{2}\KL(P_{\theta_*}\parr P_{\hat{\theta}_*})}\nonumber\\
        &=\frac{1}{2\sigma_0}\vnormf{\hat{\theta}_*-\theta_*}_2\label{eq:rhs-tv}
    \end{align}
Then by \cref{lem:LinGaus-est-norm} we know for 
\begin{align*}
\MLin
\end{align*}
bounds the RHS of \cref{eq:rhs-tv} with the corresponding probability.

\end{proof}

The following statement immediately follows.

\begin{corollary}[Frequentist Relative Simple Regret, Linear Bandits]\label{cor:frqMSR-LinGaus}
    Let $\alg$ be an $\gamma$-Monte Carlo algorithm and $\hat{\theta}_*$ be the estimated prior parameter in \cref{eq:LinGaus-theta-hat},
    and $\Alg=\alg(\theta_*)$ and $\Alg'=\alg(\hat{\theta}_*)$ be the be the oracle \alg algorithm and \alg instantiated by the estimated prior in a task after $m_0$ exploration tasks, respectively. Then for any $\epsilon$ if 
    $\MLin$,
    we know
    \begin{align}
        \E[\mu(\Ahat_{\Alg})-\mu(\Ahat_{\Alg'})]\leq 2\gamma n\epsilon 
        \left(\diam(\E_{\theta_*}[\mu])+\sigma_0 \left( 8+5\sqrt{\log\frac{|\A|}{\min(1,2\gamma n \epsilon)}}\right)\right)\nonumber
    \end{align}
    with probability $1-\delta$ for $\delta\in(2e^{-d},1)$.
\end{corollary}

Now we can bound the meta simple regret as follows.

\corfLinSR*

\begin{proof}
     First assume we have $m_0$ exploration tasks, for $m\geq m_0$.
    We decompose the frequentist meta simple regret in three terms.
    As Gaussian is a $\sigma$-Sub-Gaussian distribution, then by Hoeffding's inequality we can upper bound the simple regret of the exploration tasks as follows. We know
    \begin{align*}
        |\mu(\Ahat_{\alg_s})-\mu(\Ahat^*_s)|\leq \sqrt{\sigma_0^2\log(\frac{2}{\sqrt{\delta}})}
    \end{align*}
    with probability at least $1-\sqrt{\delta}$. Then for the rest of the tasks, we can use the following decomposition
    \begin{align*}
        \SR(m,n,P_{\theta_*})&=\sum_{s=1}^m \E_{\mu_s\sim P_{\theta_*}} [\mu_s(A_s^*)-\mu_s(\Ahat_{\alg_s})]
        \\
        &=\sum_{s=1}^m \E_{\mu_s\sim P_{\theta_*}} [\mu_s(A_s^*)-\mu_s(\Ahat_{\alg^*})]+\E[\mu_s(\Ahat_{\alg^*})-\mu_s(\Ahat_{\alg_s})]
    \end{align*}
    where $\alg^*$ is the algorithm that knows the correct prior $P_{\theta_*}$. Now by \cref{prop:cum2sim} and a problem-independent cumulative regret bound of TS for linear bandits \citep{abeille17linear}, we can bound the first term by $O(m d^{3/2}\sqrt{n}\log K/n)=O(m \frac{d^{3/2}\log K}{\sqrt{n}})$. The second term is bounded for any $\gamma$-Monte Carlo algorithm based on \cref{cor:frqMSR-LinGaus} by 
    \begin{align*}
        &\sum_{s=1}^m  2\gamma n\epsilon 
        \left(\diam(\E_{\theta_*}[\mu])+\sigma_0 \left( 8+5\sqrt{\log\frac{|\A|}{\min(1,2\gamma n \epsilon)}}\right)\right)\\
        &\leq 2 m n \gamma\epsilon 
        \left(\diam(\E_{\theta_*}[\mu])+\sigma_0 \left( 8+5\sqrt{\log\frac{|\A|}{\min(1,2\gamma n \epsilon)}}\right)\right)\;.
    \end{align*}
    with probability at least $1-\sqrt{\delta}$ if ${m_0}\geq \left(\frac{  d \log(2/\sqrt{\delta}) \sum_{i=1}^d\sigma_i^2 }{2\sigma_0\lambda_d^{4}(\sum_{i=1}^d a_i a_i\T) \epsilon^4}\right)^{1/3}$.
    Therefore, putting these together we get
    \begin{align*}
        \SR(m,n,P_*)=O\Bigg(2 m n \gamma\epsilon 
        \left(\diam(\E_{\theta_*}[\mu])+\sigma_0 \left( 8+5\sqrt{\log\frac{|\A|}{\min(1,2\gamma n \epsilon)}}\right)\right)\\ + m \frac{d^{3/2}\log K}{\sqrt{n}} + m_0 \sqrt{\sigma_0^2\log(\frac{2}{\sqrt{\delta}})}\Bigg)
    \end{align*}
    with probability $1-\delta$. Now note that $\gamma=1$ for TS.

    Note that $\epsilon\propto m_0^{-3/4}$, and we know $\sum_{s=1}^m s^{-3/4}=O(m^{1/4})$. Therefore, if the exploration continues in all the tasks, the regret bound above becomes $\tilde{O}\left(2 m^{1/4} n ~\diam(\E_{\theta_*}[\mu]) + m \frac{d^{3/2}\log K}{\sqrt{n}}\right)$.
\end{proof}

\section{Experimental Details and Further Results}\label{app:FurtherExperiments}

We used a combination of computing resources. The main resource we used is the USC Center for Advanced Research Computing (https://carc.usc.edu/). Their typical compute node has dual 8 to 16 core processors and resides on a 56 gigabit FDR InfiniBand backbone, each having 16 GB memory. We also used a PC with 16 GB memory and Intel(R) Core(TM) i7-10750H CPU. 

\cref{fig:Gaus-extra,fig:Lin-extra1} show the results for $n=20$ with $m=200$ tasks for MAB and $m=20$ tasks for the linear experiments, where $\sigma_q=1$, $\sigma_0=0.1$, and $\sigma=1$. Note that these are shorter tasks than in \cref{sec:experiment} and thus harder.

In \cref{fig:Gaus-extra}, note that increasing $K$ tightens the relative gap between \TS and \OTS as the tasks become harder and all of the algorithms act closer to each other.

For \cref{fig:Lin-extra1}, note that the gap between \OTS and \TS is more apparent than in \cref{fig:Gaus-extra}. This is probably because the prior over the mean parameter carries more information here as it determines the whole mean reward for $K$ arms using only $d$ dimensions. \fMetaSRM takes a while to outperform \TS as its estimation takes a while to converge to the true prior.

\begin{figure*}[tb]
    \centering
    \begin{minipage}{.45\textwidth}
        \includegraphics[width=\textwidth]{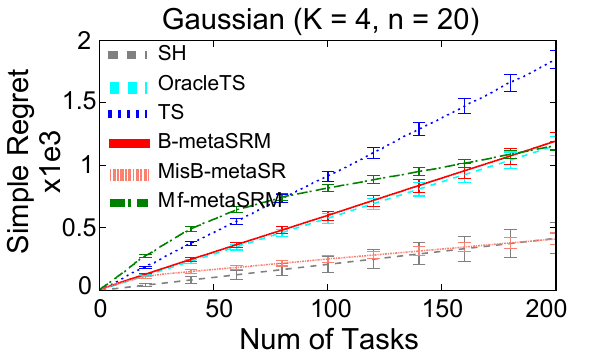}
    \end{minipage}
    \begin{minipage}{.45\textwidth}
        \includegraphics[width=\textwidth]{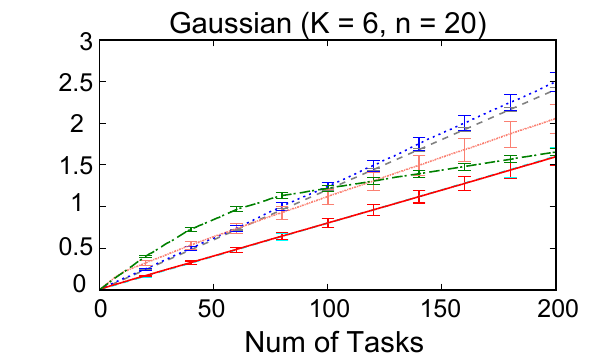}
    \end{minipage}
    \begin{minipage}{.45\textwidth}
        \includegraphics[width=\textwidth]{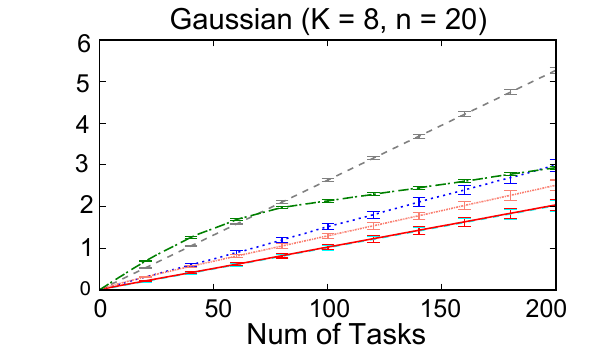}
    \end{minipage}
    \caption{Learning curves for MAB Gaussian bandit experiments. The error bars are the standard deviation of the 100 runs.}
    \label{fig:Gaus-extra}
\end{figure*}

\begin{figure*}[tb]
    \centering
    \begin{minipage}{.45\textwidth}
        \includegraphics[width=\textwidth]{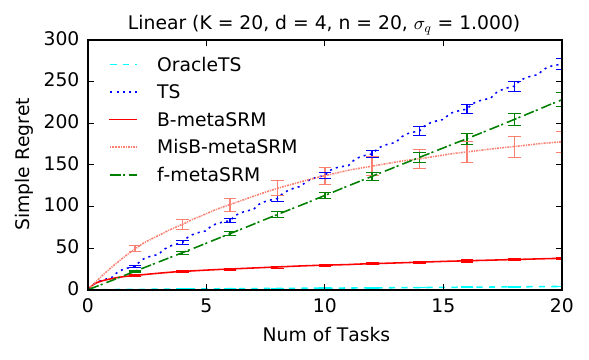}
    \end{minipage}
    \begin{minipage}{.45\textwidth}
        \includegraphics[width=\textwidth]{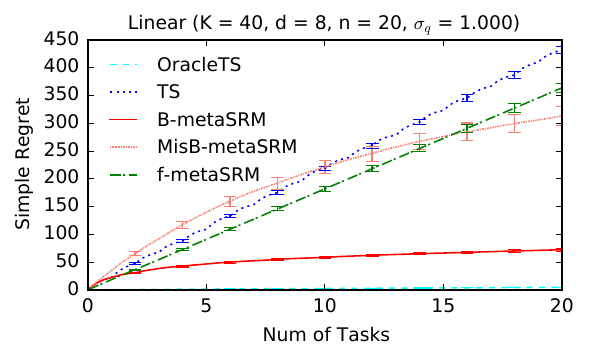}
    \end{minipage}
    \begin{minipage}{.45\textwidth}
        \includegraphics[width=\textwidth]{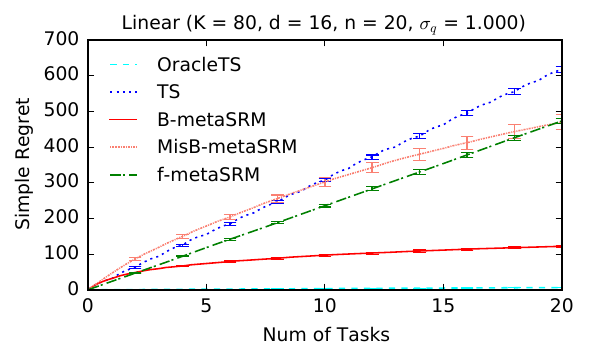}
    \end{minipage}
    \caption{Learning curves for linear Gaussian bandits experiments. The error bars are the standard deviation of the 100 runs.}
    \label{fig:Lin-extra1}
\end{figure*}

\cref{fig:LinExtra} shows further experiments for linear Gaussian bandits with larger $K$ compared to \cref{sec:LinExp}, and $K=10d$. The same setting is used.

\begin{figure}[H]
    \centering
    \begin{minipage}{.45\textwidth}
        \includegraphics[width=\textwidth]{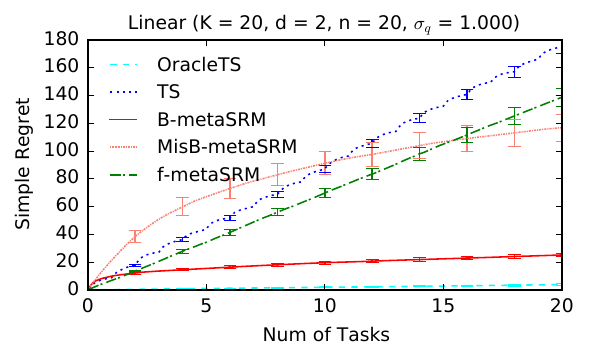}
    \end{minipage}
    \begin{minipage}{.45\textwidth}
        \includegraphics[width=\textwidth]{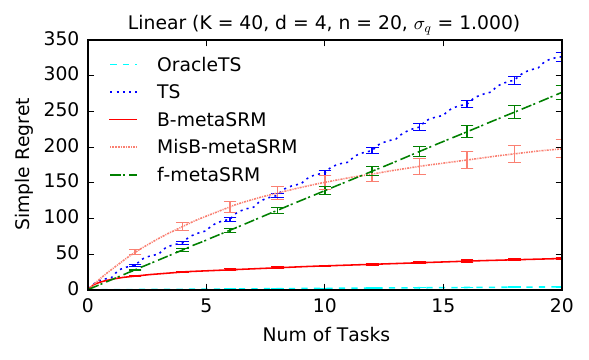}
    \end{minipage}
    \begin{minipage}{.45\textwidth}
        \includegraphics[width=\textwidth]{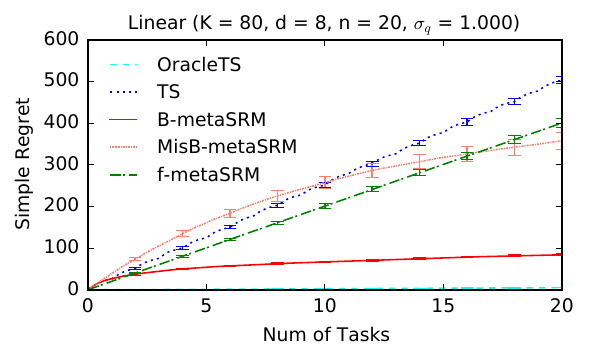}
    \end{minipage}
    \begin{minipage}{.45\textwidth}
        \includegraphics[width=\textwidth]{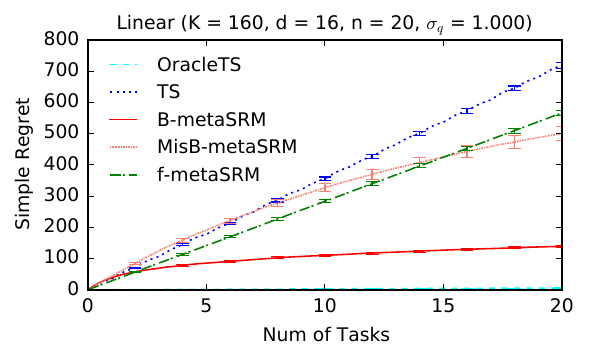}
    \end{minipage}
    \caption{Linear Gaussian bandits experiments with $K=10d$.}
    \label{fig:LinExtra}
\end{figure}

In the next experiment, we evaluate the algorithms based on their average per-task simple regret under a frequentist setting, i.e., the prior is fixed over runs. For Gaussian MAB, we use $\theta_*=[0.5, 0, 0, 0.1, 0, 0]$ with a block structured covariance so that arms 1, 2, 3 are highly correlated, and analogously for arms 4, 5, 6. The rewards are Gaussian with variance 1, which is known to all learners. For the linear case, we set the prior to be $\Ngaus(1, \Sigma_0)$ where $\Sigma_0$ is a scaled-down version of the block diagonal matrix used for the Gaussian MAB case.

\cref{fig:BML} shows the cumulative average per-task simple regret of our meta learning algorithms for Gaussian and linear Gaussian for larger number of tasks. {\tt MisTS} is a TS that uses the misspecified prior of $\Ngaus(0,I)$. Also, {\tt metaTS-SRM} is the MetaTS algorithm \citep{kveton2021metathompson} turned into a SRM algorithm. We can observe that our algorithms asymptotically achieve smaller meta simple regret over the tasks and learn the prior. Notice that \fMetaSRM has the same performance as {\tt metaTS-SRM} after convergence. This is expected as its prior estimation is updated after each task, the same as {\tt metaTS-SRM}.

\begin{figure}[H]
    \centering
    \begin{minipage}{.45\textwidth}
        \includegraphics[width=\textwidth]{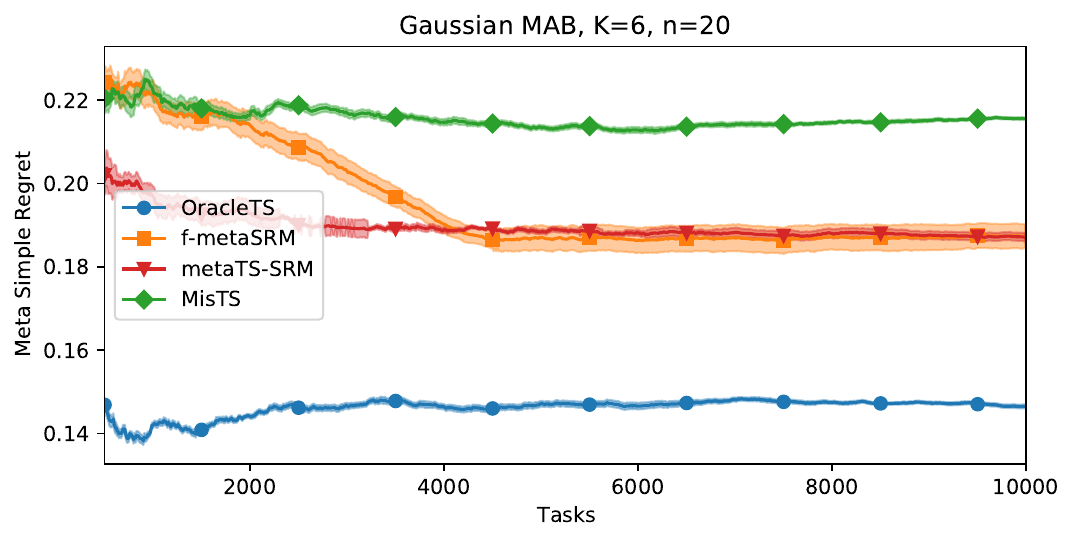}
    \end{minipage}
    \begin{minipage}{.45\textwidth}
        \includegraphics[width=\textwidth]{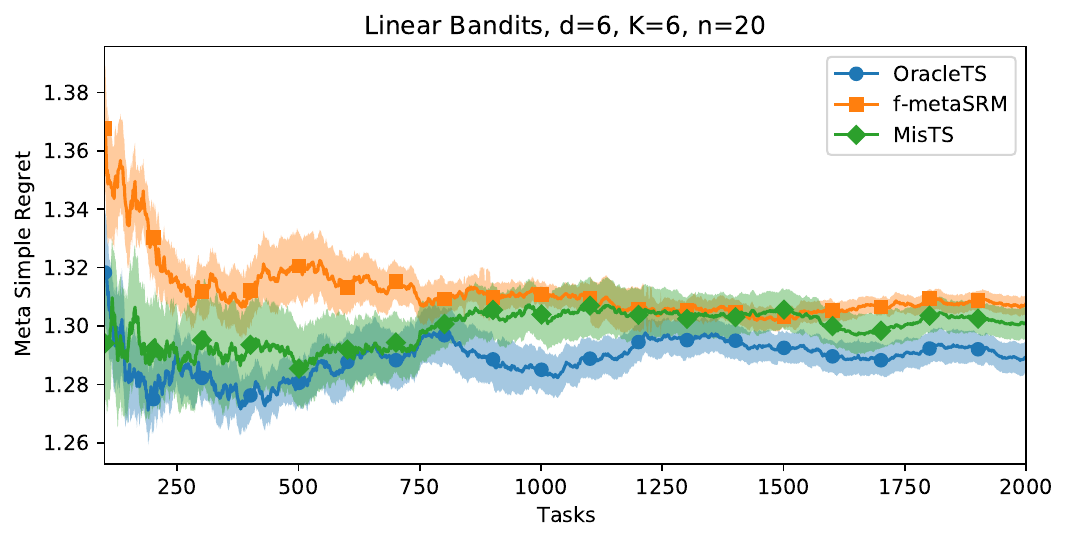}
    \end{minipage}
    \caption{Cumulative average per-task simple regret.}
    \label{fig:BML}
\end{figure}

\subsection{Real-world Experiment}\label{app:mnist}
We experimented with the MNIST\footnote{Accessed at \url{https://www.tensorflow.org/datasets/catalog/mnist}} dataset, in the same setting as in Appendix E.2 of \citet{basu2021no}. This bandit classification problem is cast as a multi-task linear bandit with Bernoulli rewards. We have a sequence of image
classification tasks where one class is selected to be positive. In each task, at every round, $K$ random images are selected as the arms and the goal is to identify the arm corresponding to an image from the positive class. The reward of an image from the selected class is Bernoulli with a mean of $0.9$. For all other classes, it is Bernoulli with mean of $0.1$.

We ran several experiments and present one representative
experiment below. When digit 0 is selected as the positive
class, the simple regret at the end of $m = 10$ tasks, each
with length n = 200, is shown in \cref{fig:mnist}.
Here $K = 30$ and the experiment is averaged over 100
random runs. We observe that \bMetaSRM outperforms the benchmarks. The BAI algorithm \LinGapE yields linear simple regret. \fMetaSRM outperforms \TS but still seems to yield linear simple regret. A larger number of tasks could help confirm this.

\cref{fig:mnist-ps} illustrates the posterior of each algorithm from the best digit (which is 0 here). As we can see \bMetaSRM quickly catches up with the \OTS and grasps the correct posterior. However, other algorithms fail to do it quickly enough and even get trapped in a false posterior. As we can see \fMetaSRM also fails which shows its linear estimation is not robust to model misspecification and underperforms in a non-linear environment like this setting.

\begin{figure}[H]
    \centering
    \begin{minipage}{.53\textwidth}
        \includegraphics[width=\textwidth]{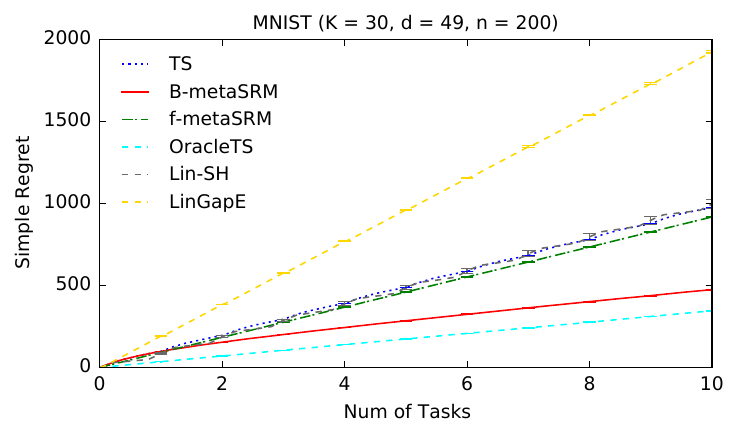}
        \caption{Cumulative average per-task simple regret for the MNIST experiment.}
        \label{fig:mnist}
    \end{minipage}
    \begin{minipage}{.4\textwidth}
        \includegraphics[width=\textwidth]{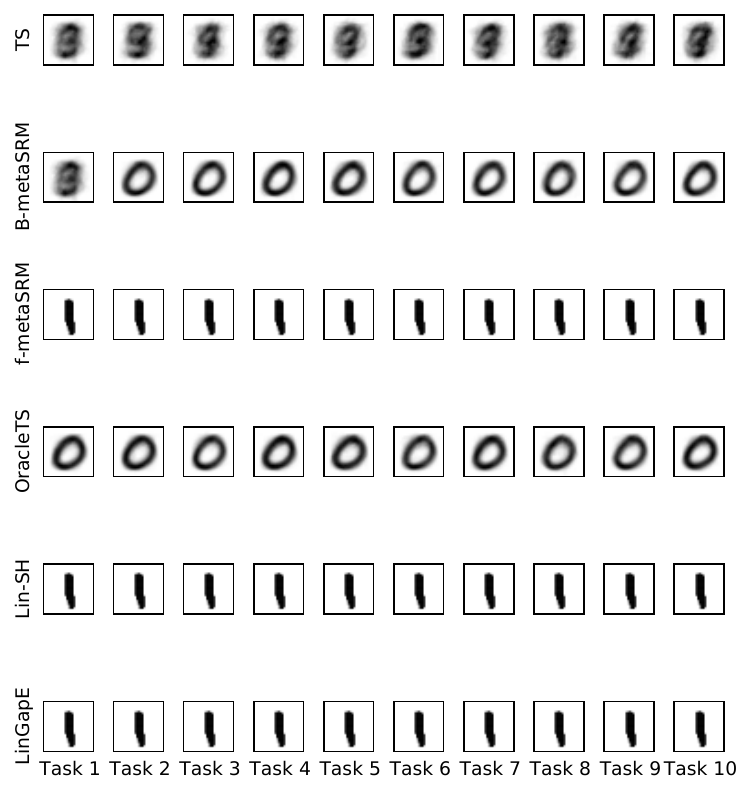}
        \caption{Cumulative average per-task simple regret for the MNIST experiment.}
        \label{fig:mnist-ps}
    \end{minipage}    
\end{figure}

\end{document}